\documentstyle[fleqn,12pt,graphicx,amssymb]{article}
\makeatletter
\newcounter{stelling}
\def\thestelling{\@arabic\c@stelling}
\newcounter{rulenr}
\def\therulenr{\@arabic\c@rulenr} 
\newenvironment{proof}{\begin{trivlist} \parindent = 0pt \parskip = 1.5ex \item[\hskip \labelsep{\bf Proof}]}{\hspace*{\fill} $\Box$ \end{trivlist}}
\newenvironment{example}{\refstepcounter{stelling} \begin{list}{{\bf Example \thestelling}}{} \item}{\end{list}}
\newenvironment{lemma}{\refstepcounter{stelling} \begin{list}{{\bf Lemma \thestelling}}{} \item}{\end{list}}
\newenvironment{definition}{\refstepcounter{stelling} \begin{list}{{\bf Definition \thestelling}}{} \item}{\end{list}}
\newenvironment{theorem}{\refstepcounter{stelling} \begin{list}{{\bf Theorem \thestelling}}{} \item}{\end{list}}
\newenvironment{property}{\refstepcounter{stelling} \begin{list}{{\bf Property \thestelling}}{} \item}{\end{list}}

\newenvironment{assumption}{\refstepcounter{stelling} \begin{list}{{\bf Assumption \thestelling}}{} \item}{\end{list}}
\newenvironment{ruledef}{\refstepcounter{rulenr} \begin{list}{{\bf Rule \therulenr}}{} \item}{\end{list}}
\newenvironment{remark}{\refstepcounter{stelling} \begin{list}{{\bf Remark \thestelling}}{} \item}{\end{list}}
\newenvironment{corollary}{\refstepcounter{stelling} \begin{list}{{\bf Corollary \thestelling}}{} \item}{\end{list}}
\renewcommand{\models}{\mid\kern-2pt=}
\newcommand{\nm}{\mid\kern-2.1pt\sim}
\renewcommand{\vdash}{\mid\kern-6pt-}
\newcommand{\sat}{\mid\kern-2pt\equiv}
\renewcommand{\emptyset}{\varnothing}
\begin{document}
	\pagenumbering{roman}

	\begin{center}
		\large
		{\bf A logic for reasoning with inconsistent knowledge \\
		{\em A reformulation using nowadays terminology (2024)}}
	\end{center}
	
	\paragraph{Changes to the original paper}
	The paper is a reformulation of the paper ``N.Roos, A logic for reasoning with inconsistent knowledge {\it Artificial Intelligence} {\bf 57} (1992) 69-103'' \cite{Roo-92} using nowadays terminology. The original paper determines `{\em justifications}' for deriving conclusion and resolving inconsistencies in provided knowledge and information. These justifications are actually {\em arguments} that are evaluated using the {\em stable semantics}, and the approach is an \emph{assumption-based argumentation system}. The current version of the paper talks arguments instead of justifications.
	
	Another change concerns the addition of superscripts to some symbols. The first paragraph of Section 4 states that a {linear extension} $\prec'$ of the reliability relation $\prec$ is considered. To make this clearer in the formalization, the superscript $^{\prec'}$ is added to anything that depends on the linear extension $\prec'$ that is currently considered.
	
	Section 9 is new and has been inserted to describe the relation with Dung's argumentation framework \cite{Dun-95}.

	The original text of the paper has not been updated except for a few typing errors and improvements to some of the proofs.
	
	\paragraph{History}
	The work on the topic described in the paper is based on the author's Master Thesis (1987) where a ranked set of premisses were used. The ranking was replaced by a partial order in 1988  \cite{Roo-88a,Roo-88b}. In the latter reports, the grounded semantics was used for drawing conclusions and inconsistencies were not resolved in the absence of a unique least preferred premiss among the premisses from which the inconsistency was derived. In 1989, the grounded semantics was replaced by the stable semantics \cite{Roo-89a,Roo-89b,Roo-89c} and in the absence of a unique least preferred premiss among the premisses from which the inconsistency is derived, every minimal premiss is considered. 
	\cite{Roo-89a} was submitted for publication to the {\it Artificial Intelligence journal} in 1989 after it was rejected for {\it IJCAI-89}. After a rather long review period, it was accepted with revisions in 1991. The AI Journal paper, on which the updated paper presented here is based, was the result of processing the reviewers recommendation.
	
	\paragraph{Related work}
	There is a correspondence between the ``logic for reasoning with inconsistent knowledge'' and Brewka's ``Preferred Subtheories'' \cite{Bre-89}. Both draw conclusions from preferred maximally consistent subsets of the premisses. The ``logic for reasoning with inconsistent knowledge'' \cite{Roo-88a,Roo-88b,Roo-89a,Roo-89b,Roo-89c} was developed independent of Brewka's ``Preferred Subtheories'' \cite{Bre-89}.
	The ``logic for reasoning with inconsistent knowledge'' differs from Brewka's ``Preferred Subtheories'' by also presenting an {\em assumption-based argumentation system} for deriving conclusions, and a {\em preferential model semantics}. 
	
\newpage
	\pagenumbering{arabic}
	\setcounter{page}{1}
\title{A logic for reasoning with inconsistent knowledge\footnote{The research reported on was carried out at 
		Delft University of Technology (TU-Delft) \cite{Roo-88a,Roo-89a} and the
		Royal Netherlands Aerospace Centre (NLR) in Amsterdam \cite{Roo-88b,Roo-89b}. Parts of the research were published in \cite{Roo-89c}} \\
	\Large{\em A reformulation using nowadays terminology (2024)}}
\author{Nico Roos}
\date{}
\maketitle
Current affiliation:
\begin{center}
	Department of Advanced Computing Sciences \\
	Faculty of Science and Engineering \\
	Maastricht University \\
	The Netherlands
\end{center}

This paper has been published in {\it Artificial Intelligence} 
{\bf 57} (1992) 69-103.

\begin{abstract}
In many situations humans have to reason with inconsistent knowledge.
These inconsistencies may occur due to not fully reliable sources of information.
In order to reason with inconsistent knowledge, it is not possible 
to view a set of premisses as absolute truths as is done in 
predicate logic.
Viewing the set of premisses as a set of assumptions, however, it is possible 
to deduce useful conclusions from an inconsistent set of premisses.
In this paper a logic for reasoning with inconsistent 
knowledge is described.
This logic is a generalization of  the work of N.~Rescher \cite{Res-64}.
In the logic a reliability relation is used to choose between 
incompatible assumptions.
These choices are only made when a contradiction is derived.
As long as no contradiction is derived, the knowledge is assumed to be 
consistent.
This makes it possible to define an {\em argumentation}-based deduction process for
the logic.
For the logic a semantics based on the ideas of Y.~Shoham \cite{Sho-86,Sho-87},
is defined.
It turns out that the semantics for the logic is a preferential semantics 
according to the definition S.~Kraus, D.~Lehmann and M.~Magidor \cite{Kra-90}.
Therefore the logic is a logic of system {\bf P} and possesses all
the properties of an ideal non-monotonic logic.
\end{abstract}


\section{Introduction}
In many situations humans have to reason with inconsistent knowledge.
These inconsistencies may occur due to sources of information which are
not fully reliable.
For example, in daylight information about the position of an object coming from
your eyes is more reliable than the information about the position of the object
coming from your ears.
But even reliable sources such as domain experts, do not always agree.

To be able to reason with inconsistent knowledge it is not possible 
to view a set of premisses as absolute truths, as in 
predicate logic.
Viewing a set of premisses as a set of assumptions, however, makes it possible 
to deduce useful conclusions from an inconsistent set of premisses.
As long as we do not have it proven otherwise, the premisses are assumed to be true
statements about the world.
When, however, a contradiction is derived, we can no longer  make this assumption.
To restore consistency, one of the premisses has to be removed.
To be able to select a premiss to be removed, a reliability relation on the premisses
will be used.
This reliability relation denotes the relative reliability of the premisses.

In the following sections I will first describe the propositional case.
After describing the propositional case, I will describe how to extend the logic to the
first order case.

\section{Basic concepts} \label{just}
The language $L$, that will be used  to express the propositions of the logic,
consists of the propositions that can be generated using a set of atomic propositions
and the logical operators $\neg$ and $\to$.
When in this paper the operators $\wedge$ and $\vee$ are used, they should 
be interpreted as shortcuts: i.e.\ $\alpha \wedge \beta$ for 
$\neg ( \alpha \to \neg \beta)$ and $\alpha \vee \beta$ for 
$\neg \alpha \to \beta$.

To be able to reason with inconsistent knowledge, I will consider
premisses to be assumptions.
These premisses are assumed to be true as long as
we do not derive a contradiction from them.
If, however, a contradiction is derived, we have to determine the premisses on which the 
contradiction is based.
The premisses on which a contradiction is based are the premisses used in the 
derivation of the contradiction.
When we know these premisses, we have to remove one of them to block 
the derivation of the contradiction.
To select a premiss to be removed, I will use a reliability relation.
This reliability relation denotes the relative reliability of the premisses.
It denotes that one premiss is more reliable than some other premiss.
Clearly the relation must be irreflexive, asymmetric and transitive.
I do not demand this  relation to be total,
for a total reliability relation implies complete knowledge about the relative
reliability of the premisses.
This does not always have to be the case.

A set of premisses $\Sigma$ is a subset of the language $L$.
On  the set of premisses $\Sigma$ a partial reliability relation
$\prec$ may be defined.
Together they form a reliability theory.
\begin{definition}
A reliability theory is a tuple $\langle \Sigma, \prec \rangle$ where
$\Sigma \subseteq L$ is a finite set of premisses and 
$\prec \; \subseteq (\Sigma \times \Sigma)$ is an irreflexive, asymmetric and transitive
partial reliability relation.
\end{definition}

Using the reliability relation, we have to remove a least preferred premiss of
the inconsistent set, thereby blocking the derivation of the
contradiction.
\begin{example}
Let $\Sigma$ denote a set of premisses,
\[\Sigma = \{ 1. \; \varphi,2. \; \varphi \rightarrow \psi,3. \; \neg \psi,4. \; \alpha \} \]
and $\prec$ a reliability relation on $\Sigma$:
\[\prec \; = \{ (3,1), (3, 2) \} \]
From $\Sigma $, $\psi$ can be derived using premisses 1 and 2.
Furthermore, a contradiction can be derived from $\psi$ and premiss 3.
Hence, the contradiction is based on the premisses 1, 2 and 3.
Since premiss 3 is the least reliable premiss on which the 
contradiction is based, it has to be removed.
\end{example}
Three problems may arise when trying to block the derivation of a contradiction.
\begin{itemize}
\item
Firstly, we have to be able to determine the premisses on which a
contradiction is based.
These are the premisses that are used in the derivation of the contradiction.
To solve this problem, \emph{supporting arguments} are introduced.
A {\it supporting argument} describes the 
premisses from which a proposition is derived.
\item
Secondly, a premiss that has been removed, may have to be placed back 
because the contradiction causing its removal is also blocked by the removal of 
another premiss.
This may occur because of some other contradiction being derived.
\begin{example}
Let $\Sigma$ be a set of premisses
\[ \Sigma = \{ \alpha, \neg \alpha \wedge \neg \beta, \beta \} \]
and let $\prec$ be a reliability relation on $\Sigma$ given by
\[ \alpha \prec (\neg \alpha \wedge \neg \beta) \prec \beta. \]
From $\alpha$ and $\neg \alpha \wedge \neg \beta$ we can derive a 
contradiction causing the removal of $\alpha$.
From $\neg \alpha \wedge \neg \beta$ and $\beta$ we can also derive a 
contradiction causing the removal of $\neg \alpha \wedge \neg \beta$.
When $\neg \alpha \wedge \neg \beta$ is removed, it is no longer necessary 
that $\alpha$ is also removed from the set of premisses to avoid the derivation of 
a contradiction.
\end{example}
To solve this problem, \emph{undermining arguments}\footnote{Sometimes the term {\em undercutting argument} is used. An undercutting argument attacks the application of a defeasible rule. Defeasible premisses can be described by defeasible rules with an empty antecedent.} are introduced. An {\it undermining argument} describes which premiss must be removed if other premisses are assumed to be true.
It is a constraint on the set of premisses we assume to be true.
\item
Thirdly, there need not exist a single least reliable premiss in
the set of premisses on which a contradiction is based.
This can occur when no 
reliability relation between premisses is specified.
In such a situation we have to consider the results of 
the removal of every alternative separately.

Choosing a premiss to be removed implies that we assume the alternatives to be 
more reliable.
Since the reliability relation is transitive, making such a choice
influences the reliability relation defined on the premisses.
\begin{example} \label{choice-to-pref}
Let $\Sigma = \{ a,b, \neg a, \neg b \}$ be a set of premisses and let \\
$\prec \; = \{ (a,\neg b), (b, \neg a) \}$ be a reliability relation on $\Sigma$.
Since $a$ and $\neg a$ are in conflict and since there is no reliability relation 
defined between them, we have to choose a culprit.
If we choose to remove $\neg a$, $a$ is assumed to be more reliable.
Therefore, $\neg b$ is more reliable than $b$.
Hence, since $b$ and $\neg b$ are also in conflict, $b$ must be removed.
\end{example}
As is illustrated in the example above, the premisses removed depend on
the extension of the reliability relation.
Therefore, in the logic described here, every (strict) linear extension of the reliability
relation will be considered.

Different linear extensions of the reliability relation can result in 
different subsets of the premisses that are assumed to be true statement about
the world ({\it that can be believed}).
The set of theorems is defined as the intersection of all extensions of the logic.
\end{itemize}

As mentioned above, two types of arguments, 
{\it supporting arguments} and 
{\it undermining arguments}, will be used.
A supporting argument is used to denote that a proposition
is believed if the premisses in the antecedent are believed, while an 
undermining argument is used to denote that a premiss can no longer 
be believed (must be withdrawn) if the premisses in the 
antecedent are believed.
\begin{definition}
Let $\Sigma $ be a set of premisses.
Then a {\em supporting argument} is a formula:
\[ P \Rightarrow \varphi \]
where $P$ is a subset of the set of premisses $\Sigma$ and
$\varphi \in L$ is a proposition. $\Rightarrow$ can be viewed as the warrant of the argument \cite{Tou-58}. \\
An {\em undermining argument} is a formula:
\[ P \not\Rightarrow \varphi \] 
where $P$ is a subset of the set of premisses $\Sigma$ and
$\varphi$ is premiss in $\Sigma$, but not in $P$. $\not\Rightarrow$ can be viewed as the warrant of the argument.
\end{definition}

\section{Characterizing the set of theorems}
In this section a characterization, based on the ideas of N.~Rescher
\cite{Res-64}, is given for the
set of theorems of a reliability theory.
As is mentioned in the previous section, linear extensions of the 
reliability relation have to be considered.
For each linear extension a set of premisses that can still be believed can 
be determined.
This set can be determined by enumerating the premisses with respect to the linear
extension of the reliability relation, starting with the most reliable
premiss.
Starting with an empty set $D$, if a premiss may consistently be added to the set $D$, it {\em should} be added.
Otherwise it must be ignored.
Because the most reliable premisses are added first, we get a {\em most reliable 
consistent set of premisses}.\footnote{The original version of this paper \cite{Roo-89a,Roo-89b} does not consider  linear extensions of the reliability relation. The partial ordering of the premisses used in \cite{Roo-89a,Roo-89b} is not a reliability relation, but a preference relation used to choose between incompatible premisses. Considering linear extensions of the reliability relation makes the approach similar to Brewka's \emph{Preferred Subtheories} \cite{Bre-89}.}
\begin{definition} \label{proof-th}
Let $\langle \Sigma,\prec \rangle$ 
be a reliability theory.
Furthermore, let $\sigma_1, \sigma_2, ..., \sigma_m$ be some enumeration
of ${\Sigma}$ such that for every $\sigma_j \prec \sigma_k $: $k < j$.
\begin{list}{}{} \item
Then $D$ is a most reliable consistent set of premisses if and only if:
\[ D = D_m ,D_0  =  \emptyset \] 
and for $0 < i < m $
\[ D_{i+1}  =  \left\{ \begin{array}{ll}
	D_i \cup \{\sigma_i \} & \mbox{if $D_i \cup \{\sigma_i \}$ 
	is consistent} \\ 
	D_i & \mbox{otherwise}
	\end{array} \right.
\]
\end{list}
\end{definition}

Let $\cal R$ be the set of all the most reliable 
consistent sets of premisses that can be determined.
\begin{definition} 
Let $\langle \Sigma,\prec \rangle$ 
be a reliability theory.

Then the set $\cal R$ of all the most reliable 
consistent sets of premisses is defined by:
\[ {\cal R} = \{ D \mid \begin{array}[t]{l} D \mbox{ is a most reliable consistent set of premisses} \\
\mbox{given some enumeration of } \Sigma \mbox{ consistent with} \prec \ \}. \end{array} \]
\end{definition}
The set of theorems of a reliability theory is defined as the 
set of those propositions that are logically entailed by every most reliable 
consistent set of premisses in $\cal R$.
\begin{definition}
Let $\langle \Sigma,\prec \rangle$ 
be a reliability theory and let $\cal R$ be the corresponding set of all the most reliable 
consistent sets of premisses.

Then the set of theorems of $\langle \Sigma,\prec \rangle$ is defined as:
\[ {\it Th}(\langle \Sigma,\prec \rangle) = \bigcap_{D \in {\cal R}} {\it Th}(D). \]
where $ {\it Th}(D) = \{ \varphi \mid D \vdash\varphi \} $
\end{definition}

\section{The deduction process} \label{deduc}
In this section a deduction process for a reliability theory is described.
\emph{Given a strict linear extension $\prec'$ of the reliability relation $\prec$},
the deduction process determines the set of premisses
that can be believed.
\begin{remark}
Instead of starting a deduction process for every strict linear extension of $\prec$,
we can also create different extensions of $\prec$ when a contradiction not based 
on a single least reliable premiss, is derived.
This approach results in one deduction tree instead of a deduction sequence
for every linear extension of $\prec$.
\end{remark}

Instead of deriving new propositions, only new 
{\em arguments} are derived.
These arguments are generated by the inference rules.
The reason why arguments instead of propositions are derived, is that
the propositions that can be believed (the belief set) depend on the set
of premisses that can still be believed.
Since this set of premisses may change because of new information derived, 
the belief set can change in a non-monotonic way.
The arguments, however, do not depend on the information derived. 
Furthermore, they contain all the information needed to determine 
the premisses that can still be believed and the corresponding belief set. 
Note that the set of arguments depends on the linear extension $\prec'$ of $\prec$ that we consider.

Starting with an initial set of arguments ${\cal A}^{\prec'}_0$, the deduction process
generates a sequence of sets of arguments:
\[ {\cal A}^{\prec'}_0, {\cal A}^{\prec'}_1, {\cal A}^{\prec'}_2, ... \]
With each set of arguments ${\cal A}^{\prec'}_i$ there corresponds a belief set $B^{\prec'}_i$.
So we get a sequence of belief sets:
\[ B^{\prec'}_0, B^{\prec'}_1, B^{\prec'}_2, ... \]
Although for the set of arguments there holds:
\[ {\cal A}^{\prec'}_i \subseteq {\cal A}^{\prec'}_{i +1} \]
such a property does not hold for the belief sets.
Because a belief set $B^{\prec'}_i$ is determined by evaluating the arguments ${\cal A}^{\prec'}_i$, the
belief set can change in a non-monotonic way.
J.~W.~Goodwin has called this the process non-monotonicity of 
the deduction process \cite{Goo-87}.
According to Goodwin this process non-monotonicity is just another aspect
of non-monotonic logics.

In the limit, when all the argument ${\cal A}^{\prec'}_\infty$ have been derived, 
the corresponding belief set $B_\infty$ will be equal to an extension 
of the reliability theory.
Goodwin  has called such this process of deriving the set of theorems, the 
{\em logical process theory} of a logic \cite{Goo-87}.
The logical process theory focuses on the deduction process of a logic.
In this it differs from the logic itself, which only focuses on derivability;
i.e.\ logics only characterize the set of theorems that follow from the premisses.

A deduction process for the logic starts with an initial
set of arguments ${\cal A}^{\prec'}_0$.
This initial set ${\cal A}^{\prec'}_0$ contains a supporting argument for every premiss.
These arguments indicate that a proposition is believed if the 
corresponding premiss is believed.
\begin{definition} 
Let $\Sigma$ be a set of premisses.
Then the set of initial arguments ${\cal A}^{\prec'}_0$ is defined as follows:
\[ {\cal A}^{\prec'}_0 = \{ \{ \varphi \} \Rightarrow \varphi \mid \varphi \in 
\Sigma  \}. \]
\end{definition}
Each set of arguments ${\cal A}^{\prec'}_i$ with $i > 0$ is generated from the set 
${\cal A}^{\prec'}_{i - 1}$ by adding a new argument.
How these arguments are determined, depends on the deduction system 
used.
In the following description of the deduction process, I will use an 
axiomatic deduction system for the language $L$, only containing the
logical operators $\rightarrow$ and $\neg$.
\begin{list}{}{} 
\item[{\bf Axioms}]
The logical axioms are the tautologies of a propositional logic.
\end{list}

Because an axiomatic approach is used, 
arguments for the axioms have to be introduced.
Since an axiom is always valid, it must have an 
supporting argument with an antecedent equal to the empty set.
An axiom is introduced by the following 
axiom rule.
\begin{ruledef} \label{rule1}
An axiom $\varphi$ gets a 
supporting argument $\emptyset \Rightarrow \varphi$. 
\end{ruledef}

In the deduction system two inference rules will be used, namely
the modus ponens and the contradiction rule.
Modus ponens introduces a new supporting argument for some proposition.
This argument is constructed from the arguments for the 
antecedents of modus ponens.
\begin{ruledef} \label{rule2}
Let $\varphi$ and $\varphi \rightarrow 
\psi$ be two propositions with arguments, respectively $P \Rightarrow \varphi$ and 
$Q \Rightarrow ( \varphi \rightarrow \psi )$.
\begin{list}{}{} \item
Then the proposition $\psi$ 
gets a supporting argument $(P \cup Q) \Rightarrow \psi$.
\end{list}
\end{ruledef}
While modus ponens introduces a new 
supporting argument, the contradiction rule introduces a new 
undermining argument to eliminate a contradiction.
\begin{ruledef} \label{rule3}
Let $\varphi$ and $\neg \varphi$ 
be propositions with arguments
$P \Rightarrow \varphi$ and $Q \Rightarrow 
\neg \varphi$ and let $\eta = \min_{\prec'}(P \cup Q)$ where the function $\it min$ 
selects the minimal element given the extended reliability relation $\prec'$.
\begin{list}{}{} \item
Then the premiss $\eta$ gets an
undermining argument $((P \cup Q) / \eta ) \not\Rightarrow \eta$.\footnote{In \cite{Roo-89b,Roo-89c}, no linear extensions of $\prec$ where considered and an undermining argument for every premiss in $\min_{\prec}(P \cup Q)$ is constructed.}
\end{list}
\end{ruledef}

In order to guarantee that the current set of believed premisses will approximate 
a most reliable consistent set of premisses,
we have to guarantee that the process creating new arguments is fair; i.e.\
the process does not forever defer the addition of some 
possible argument to the set of arguments.
\begin{assumption} \label{fair}
The reasoning process will not defer the addition of any possible argument 
to the set of arguments forever.
\end{assumption}
If a fair process is used, the following theorems hold.
The first theorem guarantees the soundness of the supporting arguments;
i.e.\ the antecedent of a supporting argument logically entails the
consequent of the supporting argument. 
The second theorem guarantees the completeness of the supporting arguments;
i.e.\ if a proposition is logically entailed by a subset of the premisses, then there
exists a corresponding supporting argument.
Finally, the third and fourth theorem guarantee  respectively the soundness
and the completeness of the undermining arguments.
\begin{theorem} \label{th1} {\it Soundness} \\
For each $i \geq 0$: 
\begin{list}{}{} \item
if $P \Rightarrow \varphi
\in {\cal A}^{\prec'}_i$, then:
\[ P \subseteq {\Sigma} \mbox{ and } P
\models \varphi. \]
\end{list}
\end{theorem}
\begin{theorem} \label{th2} {\it Completeness} \\
For each $P \subseteq {\Sigma}$: 
\begin{list}{}{} \item
if $P \models \varphi$, then there exists a $Q \subseteq P$ such that
for some $i \geq 0$: \[ Q \Rightarrow \varphi \in {\cal A}^{\prec'}_i. \] 
\end{list}
\end{theorem}
\begin{theorem} \label{th3} {\it Soundness} \\
For each $i \geq 0$: 
\begin{list}{}{} \item
if $P \not\Rightarrow \varphi 
\in {\cal A}^{\prec'}_i$, then: 
\[ (P \cup \{ \varphi \}) \subseteq 
{ \Sigma }, \mbox{ and } (P \cup \{ \varphi \}) \mbox{ is not satisfiable}. \]
\end{list}
\end{theorem}
\begin{theorem} \label{th4} {\it Completeness} \\
For each $P \subseteq { \Sigma }$: 
\begin{list}{}{} \item
if $P$ is a minimal 
unsatisfiable set of premisses and $\varphi  = \min_{\prec'}(P)$, where the function
$\it min$ selects the minimal element given the extended reliability relation $\prec'$,
then for some  
$i \geq 0$: 
\[ P \backslash \varphi \not\Rightarrow \varphi \in {\cal A}^{\prec'}_i. \]
\end{list}
\end{theorem}

Given a set of arguments, there exists a set of the 
premisses that can still be believed.
Such a set contains the premisses that do not have to be withdrawn 
because of an undermining argument.
Suppose that ${\cal A}^{\prec'}_i$ is a set of arguments derived by a reasoning agent and that
$\Delta^{\prec'}_i \subseteq {\Sigma}$ is the set of the premisses
that are assumed to be true by the reasoning agent.
Then for each premiss $\varphi$ such that for some undermining argument 
$P \not\Rightarrow \varphi \in {\cal A}^{\prec'}_i$
there holds that $P \subseteq \Delta^{\prec'}_i$,
one may not believe $\varphi$.
The set of premisses that may not be believed given a set of argument ${\cal A}^{\prec'}_i$, is 
denoted by ${\it Out}^{\prec'}_i(\Delta^{\prec'}_i)$.
\begin{definition} \mbox{}
\[ {\it Out}^{\prec'}_i (S)  =  \{ \varphi \mid P \not\Rightarrow 
\varphi \in {\cal A}^{\prec'}_i, \mbox{ and } P \subseteq S \} \]
\end{definition}

The set of premisses $\Delta^{\prec'}_i$ must, of course, be equal to the
set of premisses obtained after removing all the premisses we may not believe;
i.e.\ $\Delta^{\prec'}_i = {\Sigma} \backslash {\it Out}^{\prec'}_i(\Delta^{\prec'}_i)$.
The set of premisses that satisfy this requirement is defined by the following
fixed point definition.
\begin{definition} \label{bel.prem}
Let ${\Sigma}$ be a set of premisses and 
let ${\cal A}^{\prec'}_i$ be a set of arguments.
Then the set of premisses $\Delta^{\prec'}_i$ that can be assumed to be true, is defined as:
\[ \Delta^{\prec'}_i = {\Sigma} \backslash 
{\it Out}^{\prec'}_i ( \Delta^{\prec'}_i ). \]
\end{definition}
\begin{property}
For every $i$, the set $\Delta^{\prec'}_i$ exists and is unique.
\end{property}

After determining the set of premisses that can be believed,
the set of derived propositions that can be believed can be derived from
the supporting arguments.
This set is defined as:
\begin{definition}
Let ${\cal A}^{\prec'}_i$ be a set of arguments and $\Delta^{\prec'}_i$ be the corresponding 
set of premisses that may assumed to be true. \\
The set of propositions $B^{\prec'}_i$ that can be believed ({\em the belief set}) is defined as:
\[ B^{\prec'}_i  =  \{ \psi\mid \begin{array}[t]{l}
P \Rightarrow \psi \in {\cal A}^{\prec'}_i 
\mbox{ and } 
 P \subseteq \Delta^{\prec'}_i \}. 
\end{array} \]
\end{definition}
\begin{property} \label{belief}
$ \mbox{For each } \varphi \in B^{\prec'}_i$: $\Delta^{\prec'}_i \vdash \varphi$.
\end{property}

Let ${\cal A}^{\prec'}_\infty$ be the set of all arguments that can be derived.
\begin{definition} 
$\displaystyle {\cal A}^{\prec'}_\infty = \bigcup_{i \geq 0} {\cal A}^{\prec'}_i $
\end{definition}
The corresponding set of premisses that can be believed and 
the belief set,
will be denoted by respectively $\Delta^{\prec'}_\infty$ and by $B^{\prec'}_\infty$.
\begin{property} \label{consis}
$\Delta^{\prec'}_\infty$ is maximal consistent.
\end{property}
\begin{property} \label{Binf} \mbox{}
\[ B^{\prec'}_\infty = {\it Th}(\Delta^{\prec'}_\infty) \]
where $ {\it Th}(S) = \{ \varphi \mid S \vdash\varphi \} $
\end{property}

The following theorem implies that the characterization of the theorems of the logic,
given in the previous section,
is equivalent to the intersection of the belief sets that can be derived.
\begin{theorem}
Let $\langle \Sigma,\prec \rangle$ 
be a reliability theory.

Then there holds:
\[ {\cal R} = \{ \Delta^{\prec'}_\infty \mid \mbox{ for some linear extension $\prec'$ of $\prec$,
$\Delta^{\prec'}_\infty$ can be derived} \}. \]
\end{theorem}

\begin{corollary} \mbox{}
\[ {\it Th}( \langle \Sigma, \prec \rangle ) = \bigcap \{ B^{\prec'}_\infty \mid \mbox{ for some linear extension $\prec'$ of $\prec$} \}. \]
\end{corollary}

\section{Determination of the belief set} \label{algor}
In this section I will describe the algorithms that determine
the set of premisses that can be believed and the belief set, 
given a set of undermining arguments.
The first algorithm determines the set $\Delta^{\prec'}_i$ given the arguments ${\cal A}^{\prec'}_i$.
To understand how the algorithm works, recall that the consequent of an undermining argument
is less reliable than the premisses in the antecedent.
Therefore, if  the consequent of an undermining argument $P \not\Rightarrow \varphi$
is the most reliable premiss that can be remove, because we still belief the 
premisses in the antecedent $P$, removing $\varphi$ will never have to be undone.
After having removed $\varphi$ we can turn to the next most reliable consequent 
of an undermining argument.

The time complexity of the algorithm below depends on the {\bf for} 
and the {\bf repeat} loop.
The former loop can be executed in ${\cal O}(n)$ steps where $n$ in the number of 
undermining arguments.
The latter loop can be executed in ${\cal O}(m)$ steps where $m$ in the number of 
premisses in $\Sigma$.
Therefore, the whole algorithm can be executed in ${\cal O}(n \cdot m)$ steps.
\begin{tabbing}
mm\=mm\=mm\=mm\= \kill
{\bf begin} \\
\> $\Delta^{\prec'}_i$ := $\Sigma$; \\
\> {\bf repeat} \\
\> \> $\varphi \in {\it max}(\Sigma)$; \\
\> \> $\Sigma := \Sigma \backslash \varphi$; \\
\> \> {\bf for} each $P \not\Rightarrow \varphi \in {\cal A}^{\prec'}_i$ {\bf do} \\
\> \> \> {\bf if } $P \subseteq \Delta^{\prec'}_i$ \\
\> \> \> {\bf then} $\Delta^{\prec'}_i := \Delta^{\prec'}_i \backslash \varphi$; \\
\> {\bf until} $\Sigma = \emptyset$; \\
\> {\bf return} $\Delta^{\prec'}_i$; \\
{\bf end}.
\end{tabbing}
Using the supporting arguments, the belief set $B^{\prec'}_i$ can be determined in a straightforward way.
Clearly, $B^{\prec'}_i$ can be determined in ${\cal O}(n)$ steps where $n$ is the number of 
supporting arguments.
\begin{tabbing}
mm\=mm\=mm\= \kill
{\bf begin} \\
\> $B^{\prec'}_i = \emptyset$; \\
\> {\bf repeat} \\
\> \> $P \Rightarrow \varphi \in {\cal A}^{\prec'}_i$; \\
\> \> ${\cal A}^{\prec'}_i := {\cal A}^{\prec'}_i \backslash \{P \Rightarrow \varphi \}$; \\
\> \> {\bf if } $P \subseteq \Delta^{\prec'}_i$ \\
\> \> {\bf then} $B^{\prec'}_i := B^{\prec'}_i \cup \{ \varphi \}$; \\
\> {\bf until} ${\cal A}^{\prec'}_i = \emptyset$; \\
\> {\bf return} $B^{\prec'}_i$; \\
{\bf end}.
\end{tabbing}

\section{The semantics for the logic} \label{semantics}
The semantics of the logic is based on the ideas of 
Y.~Shoham \cite{Sho-86,Sho-87}.
In \cite{Sho-86,Sho-87} Shoham argues that the difference between 
monotonic logic and non-monotonic logic is a difference in the definition 
of the entailment relation.
In a monotonic logic a proposition is entailed by the premisses if it is true in 
every model for the premisses.
In a non-monotonic logic, however, a proposition is entailed by the premisses
if it is preferentially entailed by a set of premisses; i.e.\
if it is true in every preferred model for the premisses.
These preferred models are determined by defining an acyclic 
partial preference order on the models.

The semantics for the logic differs slightly from Shoham's approach.
Since the set of premisses may be inconsistent, the set of models for 
these premisses can be empty.
Therefore, instead of defining a preference relation on the models of 
the premisses,
a partial preference relation on the set of semantic interpretations for the language is 
defined.
Given such a preference relation on the interpretations, the models for 
a reliability theory are the most preferred semantic interpretations.
The preference relation used here is based on the following ideas.
\begin{itemize}
\item
The premisses are assumptions about the world we are reasoning about.
\item
We are more willing to give up believing a premiss with a low reliability than 
a premiss with a high reliability.
\end{itemize}
Therefore, an interpretation satisfying more premisses with a 
higher reliability $(\prec)$ than some other interpretation, 
is preferred $(\sqsubset)$ to this interpretation. 
\begin{example}
Let $\cal M$ and $\cal N$ be two interpretations.
Furthermore , let $\cal M$ satisfy $\alpha$ and $\beta$, and let 
$\cal N$ satisfy $\beta$ and $\gamma$.
Finally let $\alpha$ be more reliable than $\gamma$, $\gamma \prec \alpha$.
Clearly, we cannot compare $\cal M$ and $\cal N$ using the premiss $\beta$.
$\cal M$ and $\cal N$ can, however, be compared using the premisses $\alpha$ and $\gamma$.
Since $\alpha$ is more reliable than $\gamma$, since $\cal N$ does not 
satisfy $\alpha$ and since $\cal M$ does not satisfy $\gamma$,
$\cal M$ must be preferred to $\cal N$,
\end{example}

\begin{definition} \label{inter}
An interpretation $\cal M$ is a set containing the atomic propositions that 
are true in this interpretation.
\end{definition}
\begin{definition} \label{prem}
Let $\cal M$ be a semantic interpretation and let ${\Sigma}$
be a set of premisses. \\
Then the premisses ${\it Prem}({\cal M}) \subseteq {\Sigma}$
that are satisfied by $\cal M$, are defined as:
\[ {\it Prem}({\cal M}) = \{ \varphi \mid \varphi \in {\Sigma} 
\mbox{ and } {\cal M} \models \varphi \} \]
\end{definition}
\begin{definition} \label{prefM}
Let $\langle {\Sigma},\prec \rangle$ be a reliability theory.
Furthermore, let $\sqsubset$ be a preference relation on the interpretations. \\
For every interpretation $\cal M, N$ there holds:
\begin{list}{}{}
\item
$ {\cal M} \sqsubset {\cal N}$ if and only if ${\it Prem}({\cal M}) \not= {\it Prem}({\cal 
N})$ and for every \\
$\varphi \in ( {\it Prem}({\cal M}) \backslash {\it Prem}({\cal N}) )$, there is a 
$\psi \in ( {\it Prem}({\cal N}) \backslash {\it Prem}({\cal M}) )$ such that: 
\[ \varphi \prec \psi. \] 
\end{list}
\end{definition}
The preference relation on the interpretations has the following property:
\begin{property} \label{prop:irref-trans}
	Let $\langle {\Sigma},\prec \rangle$ be a reliability theory and let $\sqsubset$ be the preference relation over interpretations defined the reliability theory.
	
	$\sqsubset$ is irreflexive and transitive.
\end{property}

Given the preference relation on the interpretations, the set of models
for the premisses can be defined.
\begin{definition}
Let $\langle {\Sigma},\prec \rangle$ be a reliability theory and 
let ${\it Mod}_\sqsubset (\langle {\Sigma},\prec \rangle)$ denote the models 
for the reliability theory. 
\begin{list}{}{} \item
$ {\cal M} \in {\it Mod}_\sqsubset (\langle {\Sigma},\prec \rangle) $  if and only if 
there exists no interpretation $ {\cal N}$ such that: \[ {\cal M} \sqsubset {\cal N}. \]
\end{list}
\end{definition}
Now the following important theorem, guaranteeing the soundness
and the completeness of the logic, holds:
\begin{theorem} \label{s.c}
Let $\langle {\Sigma},\prec \rangle$ be a reliability theory.
Furthermore, let ${\cal R}$ be the corresponding set of 
all most reliable consistent sets of premisses.
Then:
\[ {\it Mod}_\sqsubset (\langle {\Sigma},\prec \rangle) = \bigcup_{\Delta^{\prec'}_\infty \in {\cal R}} 
{\it Mod}( \Delta^{\prec'}_\infty) \]
where ${\it Mod}(S)$ denotes the set of classical models for a set of propositions $S$. 
\end{theorem}

\section{Some properties of the logic}
In this section I will discuss some properties of the logic.
Firstly, I will relate the logic to the general framework for 
non-monotonic logics described by S.~Kraus, D.~Lehmann and 
M.~Magidor \cite{Kra-90}.
Secondly, I will compare the behaviour of the logic when new information is added
with G\"ardenfors's theory for belief revision \cite{Gar-88}.

\subsection{Preferential models and cumulative logics}
In \cite{Kra-90} Kraus et al. describe a general framework for the study of non-monotonic 
logics.
They distinguish five general logical systems and show how each of them
can be characterized by the properties of the consequence relation.
Furthermore, for each consequence relation a different class of models is defined.
The consequence relations and the classes of models are related to each other by
representation theorems.

The consequence relation relevant for the logic discussed here is the preferential 
consequence relation of system {\bf P}.
I will show that the preference relation on the semantic interpretations, described 
in the previous section, corresponds to a preferential model described by 
Kraus et al.
\begin{lemma} \label{lem}
Let $\langle {\Sigma},\prec \rangle$ be a reliability theory.
Furthermore, let $\widehat{\alpha} = \{ {\cal M} \mid {\cal M} \models \alpha \} $,
let $\Sigma' = \Sigma \cup \{ \alpha \}$ and 
let $\prec' \; = ( \prec \ \cap \; ( \Sigma \backslash \alpha \times 
\Sigma \backslash \alpha ) )  \cup \{ \left< \varphi, \alpha \right> \mid
\varphi \in \Sigma \backslash \alpha \}$.

Then $ {\cal M} \in {\it Mod}_{\sqsubset'} (\langle {\Sigma'},\prec' \rangle) $
if and only if
$ {\cal M} \in \widehat{\alpha} $
and for no ${\cal N} \in \widehat{\alpha}$: \[ {\cal M \sqsubset N}. \]
\end{lemma}
%
%

\begin{theorem} \label{modP}
	Let $\langle {\Sigma},\prec \rangle$ be a reliability theory. Moreover, let $\langle S,l,< \rangle$ be a triple where the set of states $S$ is the set of all possible interpretations for the language $L$, where $l : S \to S$ is the identity function, and where for each
	${\cal M,N} \in S$: 
	\begin{list}{}{}
		\item
		$\cal M < N$ if and only if $\cal N \sqsubset M$.
	\end{list}
	
	Then $\langle S,l,< \rangle$ is a \textit{preferential model} \cite{Kra-90}.
\end{theorem}

Now I will relate the consequence relation of system {\bf P} to the logic.
To motivate the relation I will describe below, recall that $\alpha \nm \beta$
should be interpreted as: `if $\alpha$, normally $\beta$'.
Hence, if we assume $\alpha$, we must assume that $\alpha$ is true beyond any doubt.
To realize this, we must add $\alpha$ as a premiss. 
Furthermore, $\alpha$ must be more reliable than any other premiss, 
otherwise we cannot guarantee that $\alpha$ is an element of the set 
of theorems ${\it Th}(\langle \Sigma, \prec \rangle)$.
It is possible that $\alpha$ is an element of the original  set of premisses.
In that case we should use the most reliable knowledge source for
a premiss; i.e.\ the assumption that $\alpha$ is true beyond any doubt.
If $\alpha$ is indeed an element of $B_\infty$, we must prove that $\beta$ 
will also be an element of ${\it Th}(\langle \Sigma, \prec \rangle)$.
\begin{theorem}
Let $W= \left< S,l,< \right>$ be a preferential model for $\langle {\Sigma},\prec \rangle$.
Then the following equivalence holds:
\begin{list}{}{} \item
$\alpha \nm_W \beta$ if and only if 
\[ \Sigma' = \Sigma \cup \{ \alpha \}, \] 
\[ \prec' \; = ( \prec \ \cap \; ( \Sigma \backslash \alpha \times 
\Sigma \backslash \alpha ) ) \cup \{ \left< \varphi, \alpha \right> \mid
\varphi \in \Sigma \backslash \alpha \} \] 
and $\beta \in {\it Th}(\langle \Sigma', \prec' \rangle)$.
\end{list}
\end{theorem}
\begin{corollary}
Let $W= \left< S,l,< \right>$ be a preferential model for $\langle {\Sigma},\prec \rangle$. \\
Then: \[ {\it Th}(\langle \Sigma, \prec \rangle) = \{ \alpha \mid \mbox{} \nm_W \alpha \} \]
\end{corollary}

\subsection{Belief revision}
In \cite{Gar-88}, G\"ardenfors  describes three different ways in which a 
belief set can be revised, viz.\ {\it expansion, revision} and {\it contraction}.
Expansion is a simple change that follows from the addition of a new proposition.
Revision is a more complex form of adding a new proposition. Here the belief set 
must be changed in such a way that the resulting belief set is consistent.
Contraction is the change necessary to stop believing some proposition.
For each of these forms of belief revision, G\"ardenfors has formulated a set of {\it
rationality postulates}.

In this subsection I will investigate which of the postulates are satisfied by the logic.
To be able to do this, the set of theorems of a reliability theory is identified 
as a belief set as defined by G\"ardenfors.
Here expansion, revision and contraction of the belief set $K$, with respect to the proposition
$\alpha$, will be denoted by respectively: $K^+ [\alpha]$, $K^*[\alpha]$ 
and $K^-[\alpha]$.

\subsubsection*{Expansion}
To expand a belief set $K$ with respect to a proposition $\alpha$, $\alpha$ should be 
added to the set of premisses that generate the belief set.
Since the logic does not allow an inconsistent belief set, 
$\alpha$ can be added if the belief set does not already contain $\neg \alpha$.
Otherwise, the logic would start revising the belief set.
Adding $\alpha$ to the set of premisses, however, is not sufficient to 
guarantee that $\alpha$ will belong to the new belief set.
Take for example the following reliability theory.
\[ \Sigma = \{ 1: \alpha \wedge \beta, 2: \neg \alpha \wedge \beta, 
3: \alpha \wedge \neg \beta, 4: \neg \alpha \wedge \neg \beta \} \]
\[ \prec \; = \{ (3, 2), (4, 1) \} \]
Clearly, adding $\alpha$ to $\Sigma$ does not result in believing $\alpha$.
Hence, the second postulate of expansion is not satisfied.
To guarantee that $\alpha$ belongs to the new belief set, we have to prefer
$\alpha$ to any other premiss.
If, however, we prefer $\alpha$ to every other premiss in the example above,
the third postulate for expansion will not be satisfied.
Hence, expansion of a belief set is not possible in the logic.
The reason for this is that the reasons for believing a proposition in a belief set
are not taken into account by the postulates for expansion.
Because of this internal structure, revision instead of expansion takes place.

\subsubsection*{Revision}
For revision of a  belief set $K$ with respect to a proposition $\alpha$,
we have to add $\alpha$ as a premiss and prefer it to any other premiss.
With this implementation of the revision process, some of the postulates for 
revision of the belief set with respect to $\alpha$ are satisfied.
The postulates not being satisfied relate revision to expansion.
Expansion, however, is not defined for the logic.
\begin{theorem}
Let belief set $K = {\it Th}(\langle \Sigma, \prec \rangle)$ be the 
set of theorems of the reliability theory 
$\langle \Sigma, \prec \rangle$. \\
Suppose that $K^*[\alpha]$ is the belief set of the premisses 
$\Sigma \cup \{ \alpha \}$ with 
reliability relation: \[ \prec' \; = ( \prec \cap \ ( \Sigma \backslash \alpha \times 
\Sigma \backslash \alpha ) ) \cup \{ \left< \varphi, \alpha \right> \mid
\varphi \in \Sigma \backslash \alpha \}; \]
i.e.\ $K^*[\alpha] = \{ \beta \mid \alpha \nm_W \beta \}$ where $W$
is a preferential model for 
$\langle {\Sigma},\prec \rangle$. \\
Then the following postulates are satisfied.
\begin{enumerate}
\item
$K^*[\alpha]$ is a belief set.
\item
$\alpha \in K^*[\alpha]$.
\setcounter{enumi}{5}
\item
If $\vdash \alpha \leftrightarrow \beta$, then $K^*[\alpha] = K^*[\beta]$.
\end{enumerate}
\end{theorem}

\subsubsection*{Contraction}
It is not possible to realise contraction of a belief set in the logic in a straight
forward way.
To be able to contract a proposition $\alpha$ from a belief set $K$, we have to 
determine the premisses on which belief in this proposition is based.
This information can be found in the applicable supporting argument that 
supports the proposition $\alpha$.
When we have determined these premisses, we have to remove some of them.
I.e. for each linear extension of the reliability relation, we must add the following 
undermining arguments to ${\cal A}^{\prec'}_\infty$
\[ \{ P\backslash\varphi \not\Rightarrow \varphi \mid P \Rightarrow \alpha \in {\cal A}^{\prec'}_\infty, 
\varphi = \min_{\prec'}(P) \}. \]
Unfortunately, this solution, which requires a modification of the logic,
can only be applied after ${\cal A}^{\prec'}_\infty$ has been determined.
Furthermore, only the most trivial postulates 1, 3, 4 and 6
will be satisfied.

\section{Extension to first order logic}
The logic described in the previous sections can be extended to a first order logic.
To realize this we have to replace the propositional language $L$ by
a first order language, which only contains the logical operators $\neg$ 
and $\to$, and the quantifier $\forall$.
Furthermore we have to replace the logical axioms for a propositional
logic by the logical axioms for a first order logic with the modus ponens 
as the only inference rule.
We can for example use the following axiom scheme, which originate from \cite{End}.
\begin{list}{}{} 
\item[{\bf Axioms}]
Let $\varphi$ be a generalization of $\psi$ if and only if for some $n \geq 0$
and variables $x_1,...,x_n$:
\[ \varphi = \forall x_1,...,\forall x_n \; \psi. \]
Since this definition includes the case $n=0$, any formula is a generalization of itself.

The logical axioms are all the generalizations of the formulas described by the following
schemata.
\begin{enumerate} 
\item Tautologies.
\item
$\forall x \varphi (x) \rightarrow \varphi (t)$
where $t$ is a term containing no variables that occur in $\varphi$.
\item
$\forall x ( \varphi  \rightarrow \psi ) \rightarrow ( \forall x 
\varphi \rightarrow \forall x \psi )$.
\item
$\varphi \rightarrow \forall x \varphi$ where $x$ does not occur 
in $\varphi$.
\end{enumerate}
\end{list}
Finally, we have to replace the definition of the semantic interpretations by 
a definition for the semantic interpretations of a first order logic.

When these modification are made we have a first order logic for reasoning 
with inconsistent knowledge.
For this first order logic all the results that can be found in the preceding section
also hold.

\section{Argumentation framework (a new section)}
Dung \cite{Dun-95} observed that
argumentation systems proposed in the literature, use the same
types of semantics and that these semantics can be studied
independent of the underlying argumentation system. He also showed
that several other forms of non-monotonic reasoning can be
reformulated as argumentation systems. 
To study the different semantics of argumentation systems independent of the underlying argumentation system, he introduces the notion of an {\em argumentation framework}.

An argumentation framework is a tuple%
\[ \langle \mathcal{A},\longrightarrow \rangle \] %
where $\mathcal{A}$ is a set of arguments and $\longrightarrow \
\subseteq \mathcal{A} \times \mathcal{A}$ is a attack relation
over the arguments. The relation $\longrightarrow$ denotes for
every $(A,B) \in \ \longrightarrow$ that the argument $B$ cannot
be valid if $A$ is valid. How the arguments
$\mathcal{A}$ and the attack relation over the arguments are
derived and what is supported by the arguments is not taken into
consideration. Note that instead of $(A,B) \in \ \longrightarrow$,
below I will use the infix notation $A \longrightarrow B$.

Section \ref{just} defined two types of arguments. Only the undermining arguments of the form $P \not\Rightarrow \varphi$ attack other arguments. Given an set of arguments ${\cal A}^{\prec'}_i$, an argumentation framework $\langle {\cal A}^{\prec'}_i,\longrightarrow^{\prec'}_i \rangle$ can be defined. 
\begin{definition}
	Let ${\cal A}$ be a set of arguments. 
	
	$\langle \mathcal{A},\longrightarrow \rangle$ is a corresponding argumentation framework where  $A \longrightarrow B$ iff 
	\begin{itemize}
		\item$\{ A,B \} \subseteq {\cal A}$, 
		\item
		$A= (P \not\Rightarrow \varphi)$, 
		\item
		$B = (Q \not\Rightarrow \psi)$ or $B = (Q \Rightarrow \psi)$, and 
		\item
		$\varphi \in Q$.
	\end{itemize}
\end{definition}
Note that the argumentation framework is an instance of {\it assumption-based argumentation} \cite{DungKT09}.

Because we are considering a linear extension $\prec'$ of $\prec$, which is a total ordering of $\Sigma$,
there is a unique stable argument extension ${\cal E}^{\prec'}_i$ for an argumentation framework $\langle {\cal A}^{\prec'}_i,\longrightarrow^{\prec'}_i \rangle$, which is also the unique grounded extension. This stable  extension ${\cal E}^{\prec'}_i$ determines the premisses $\Delta^{\prec'}_i$ that can be assumed to be true:
\[ \Delta^{\prec'}_i = \Sigma \backslash \{ \varphi \mid P \not\Rightarrow \varphi \in {\cal E}^{\prec'}_i \}\]
as well as the belief set:
\[ B^{\prec'}_i = \{ \varphi \mid P \Rightarrow \varphi \in {\cal E}^{\prec'}_i \}\]

In the original version \cite{Roo-89a,Roo-89b} of the AI journal paper \cite{Roo-92}, no linear extensions $\prec'$ of $\prec$ was considered. Instead, whenever two supporting arguments support a proposition and its negation, for every least preferred supporting premiss given the partial order $\prec$ on the premisses, an undermining argument is formulated.
\begin{ruledef} \label{rule4}
	Let $\varphi$ and $\neg \varphi$ 
	be propositions with arguments
	$P \Rightarrow \varphi$ and $Q \Rightarrow 
	\neg \varphi$.
	
	Let $\eta \in \min_\prec (P \cup Q)$ be a minimal element given the reliability relation $\prec$.
	\begin{list}{}{} \item
		Then the premiss $\eta$ gets an
		undermining argument $((P \cup Q) / \eta ) \not\Rightarrow \eta$.
	\end{list}
\end{ruledef}
So, if there is no unique least preferred premiss in $P \cup Q$ given $\prec$, multiple undermining arguments are formulated. Moreover, the stable semantics may result in multiple argument extensions. 
Some of these argument extensions may determine a set of premisses $\Delta^{\prec'}_i$, which can be assumed to be true, but give rise to the problem illustrated in Example \ref{choice-to-pref}. That is, selecting a minimal element in $\min_\prec (P \cup Q)$ introduces a preference, and all the introduced preferences combined with $\prec$ do not correspond to any linear extension of $\prec$ because the combination contains cycles. 

Instead of considering all linear extensions of $\prec$ as was described in Section 4, we can also apply Rule \ref{rule4} and determine all stable argument extensions. Some of these stable extensions may not correspond to a linear extension of $\prec$ and have to be ignored. An argument extension has to be ignored if $\prec$ together with the additional preferences introduced by the argument extension contains a cycle. Formally:
\begin{definition}
	Let $\prec$ be a partial order on the premisses $\Sigma$, and let $\cal E$ be an argument extension.
	
	The argument extension $\cal E$ must be {\em ignored} if and only if the partial order
	\[ \prec^* \; = \; \prec \cup \ \{ (\varphi,\psi) \mid P \not\Rightarrow \varphi \in {\cal A}, \psi \in P \} \]
	over $\Sigma$ contains a cycle.
\end{definition}
Note that there exists at least one linear extension of $\prec^*$ if $\prec^*$ contains no cycles. This approach is more efficient than considering all linear extensions of $\prec$, which has a worst case time complexity that is factorial in the number of premisses $\Sigma$.

\section{Related work}
In this section I will discuss some related approaches.
Firstly, the relation with of N. Rescher's work will be discussed.
Rescher's work is closely related to the logic described here.
Secondly, the relation with Poole's framework for default reasoning,
which is a special case of Rescher's work, will be discussed.
Thirdly, the difference between paraconsistent logics and the logic
described here, will be discussed.
Finally the relation with Truth Maintenance Systems, and especially 
J. de Kleer's ATMS will be  discussed.

\subsection{Hypothetical reasoning}
In his book `Hypothetical Reasoning', Rescher describes  how to reason with
an inconsistent set of premisses \cite{Res-64}.
He introduces his reasoning method, because he wants to formalize
hypothetical reasoning.
In particular, he wants to formalize reasoning with belief contravening
hypotheses, such as counterfactuals.
In the case of counterfactual reasoning, we make an assumption that 
we know to be false.
For example, let us suppose that Plato had lived during the middle ages.
To be able to make such a counter factual assumption, we, temporally, have to give up
some of our beliefs to maintain consistency.
It is, however, not always clear which of our beliefs we have to give up.
The following example gives an illustration.
\begin{example}
\hspace*{\fill} \\ \vspace{-6mm}
\begin{description}
\item[Beliefs]
\hspace*{\fill} \\ \vspace{-6mm}
\begin{enumerate}
\item
Bizet was of French nationality.
\item
Verdi was of Italian nationality.
\item
Compatriots are persons who share the same nationality.
\end{enumerate}
\item[Hypothesis]
Assume that Bizet and Verdi are compatriots.
\end{description}
\end{example}
There are three possibilities to restore consistency.
Clearly, we do not wish to witdraw 3, but we are indifferent whether we should
give up 1 or 2.

To model this behaviour in a logical system, Rescher divides the set of premisses into
modal categories.
The modalities Rescher proposes are: alethic modalities, epistemic modalities,
modalities based on inductive warrant, and modalities based on probability or 
confirmation.
Given a set of modal categories, he selects Preferred Maximal Mutually-Compatible subsets
(PMMC subsets) from them.
The procedure for selecting these subsets is as follows:
\begin{list}{}{}
\item
Let $M_0,...,M_n$ be a family of modal categories.
\begin{enumerate}
\item
Select a maximal consistent subset of $M_0$ and let this be the set $S_0$.
\item 
Form $S_i$ by adding as many premisses of $M_i$ to $S_{i-1}$ as possible 
without disturbing the consistency of $S_i$.
\end{enumerate}
$S_n$ is a PMMC-subset.
\end{list}
Given these PMMC-subsets, Rescher defines two entailment relations.
\begin{itemize}
\item
Compatible-Subset (CS) entailment. A proposition is CS entailed if it follows from every
PMMC-subset.
\item
Compatible-Restricted (CR) entailment.
A proposition is CR entailed if it follows from some PMMC-subset.
\end{itemize}

It is not difficult to see that Rescher's modal categories can be represented
by a partial reliability relation on the premisses.
For every modal category $M_i$, $M_j$ with $i <j$, there must hold that 
each premiss in $M_i$
is more reliable than any premiss in $M_j$.
Given this ordering, from Definition \ref{proof-th} it follows immediately
that the PMMC-subsets are equal to the most reliable consistent sets of premisses.

\subsection{A framework for default reasoning}
The central idea behind Poole's approach is that default reasoning should be viewed 
as {\it scientific theory formation} \cite{Poo-88}.
Given a set of facts about the world and a set of hypotheses, a subset of the hypotheses
which together with the facts can explain an {\it observation},
have to be selected.
Of course, this selected set of hypotheses has to be consistent with the facts.
A default rule is represented in Poole's framework by a hypothesis containing free
variables.
Such a hypothesis represents a set of ground instances of the hypothesis.
Each of these ground instances can be used independently of the other instances
in an explanation.
An explanation for a proposition $\varphi$ is a maximal (with respect to the
inclusion relation) {\em scenario} that implies $\varphi$.
Here a scenario is a consistent set containing all the facts and some ground instances
of the hypotheses.

This framework can be viewed as a special case of Rescher's work.
Poole's framework consists of only two modal categories, the facts $M_0$ and
the hypotheses $M_1$.
Since Rescher's work is a special case of the logic described in this paper,
so is Poole's framework.
Poole, however, extends his framework with constraints.
These constraints are added to be able to eliminate some scenarios as 
possible explanations for a formula $\varphi$.
A scenario is eliminated when it is not consistent with the constraints.

The constraints can be interpreted as describing that some scenarios
are preferred to others.
Since in the logic described in this paper a reliability relation on the premisses 
generates a preference
relation on consistent subsets of the premisses, an obvious question is
whether 
the preference relation described by the constraints can be modelled with an 
appropriate reliability relation.
Unfortunately, the answer is `no'.
This is illustrated by the following example.
\begin{example} \label{poole}
\hspace*{\fill} \\ \vspace{-6mm}
\begin{description}
\item[Facts:] $\varphi, \psi$.
\item[Defaults:] $\varphi \rightarrow \alpha, \varphi \rightarrow \neg \beta, 
\psi \rightarrow  \neg \alpha, \psi \rightarrow \beta$.
\item[Constraints:] $\neg ( \alpha \wedge \beta ), 
\neg ( \neg \alpha \wedge \neg \beta )$. 
\end{description}
Without the constraints this theory has four different extensions. 
These extensions are the logical 
consequences of the following scenarios.
\[ S_1 = \{ \varphi, \psi, \varphi \rightarrow \alpha, \varphi \rightarrow \neg \beta \} \]
\[ S_2 = \{ \varphi, \psi, \psi \rightarrow  \neg \alpha, \psi \rightarrow \beta \} \]
\[ S_3 = \{ \varphi, \psi, \varphi \rightarrow \alpha, \psi \rightarrow \beta \} \]
\[ S_4 = \{ \varphi, \psi, \varphi \rightarrow \neg \beta, \psi \rightarrow  \neg \alpha \} \]
Only the first two scenarios are consistent with constraints.
If this default theory has to be modelled in the logic, 
a reliability relation
has to be specified in such a way that $\{ S_1, S_2 \} ={\cal R}$.
To determine the required reliability relation on the hypotheses, combinations of two scenarios are considered.
To ensure that $S_1 \in {\cal R}$  and $S_3 \not\in {\cal R}$, 
$\varphi \rightarrow \neg \beta$
has to be more reliable than $\psi \rightarrow \beta$.
To ensure that $S_2 \in {\cal R}$ and $S_4 \not\in {\cal R}$,
$\psi \rightarrow \beta$
has to be more reliable than $\varphi \rightarrow \neg \beta$.
Hence, the reliability relation would be reflexive, violating the requirement
of irreflexivity in a strict partial order.
This means that not every ordering of explanations in Poole's framework
can be modelled, using the logic described in this paper.
\end{example}

Although Poole's framework without constraints
can be expressed in the logic described in this paper,
the philosophies behind the two approaches are 
quite different.
Poole's work is based on the idea that default reasoning is a process
of selecting consistent sets of hypotheses, which can explain a set of 
observations.
The purpose of the logic described in this paper, however, is to derive useful conclusions
from an inconsistent set of premisses.

\subsection{Paraconsistent logics}
Paraconsistent logics are a class of logics developed for reasoning 
with inconsistent knowledge \cite{Arr-80}.
Unlike classical logics, in paraconsistent logics there need not hold $\neg ( \varphi
\wedge \neg\varphi )$ for some proposition $\varphi$.
Hence, an inconsistent set of premisses is not equivalent to the trivial theory; it does not
imply the set of all propositions.

Unlike the logic described in this paper, a paraconsistent logic does not resolve
an inconsistency.
Instead it simply avoids that everything follows from an inconsistent theory.
To illustrate this more clearly, consider the following a reliability theory, 
without a reliability relation.
\[ \Sigma = \{ \alpha \wedge \beta, \neg \beta \wedge \gamma \} \]
In the logic described in this paper, all maximal consistent subsets will 
be generated.
\[ \{ \alpha \wedge \beta \} \mbox{ and } \{ \neg \beta \wedge \gamma \} \]
In a paraconsistent logic the proposition $\beta$ will be contradictory but 
the propositions $\alpha$ and $\gamma$ will consistently be entailed by the premisses.

The difference between the two approaches can be interpreted as the difference 
between a credulous and a sceptical view of knowledge sources.
With a credulous view of a knowledge source, we try to derive as much as is 
consistently possible.
According to Arruda \cite{Arr-80}, scientific theories for different domains, which
conflict with each other on some overlapping aspect, are treated in this way.
With a sceptical view of a knowledge source, we only believe one of the knowledge
sources that support the conflicting information.
So if a part of someone statement turns out to be wrong, we will not belief the rest of
his/her statement.
Although a credulous view of knowledge sources seems to be acceptable for scientific 
theories for different domains, a sceptical view seems to be better for knowledge 
based systems, which have to act on the information available.

\subsection{Truth maintenance systems}
In the here presented logic arguments are used. 
Justifications in the JTMS of J.~Doyle \cite{Doy-79}
or the ATMS of J.~de~Kleer \cite{Kle-86} have a similar function as arguments.
Unlike the arguments used here, these justifications are not part of the deduction process.
The arguments used here follow directly from the requirement for the
deduction process (Section \ref{just}).
Moreover, in a(n) (A)TMS the justifications describe dependencies between propositions,
while in the here presented logic, the supporting arguments describe dependencies between
propositions and premisses, and undermining arguments describe dependencies among premisses.
The supporting arguments of the  logic, however, can be compared with the
labels in the ATMS \cite{Kle-86}.
Like a label, a supporting argument describes from which premisses a proposition is 
derived.
The undermining arguments have more or less the same function as a {\bf nogood}
in the ATMS.
As with an element from the set representing a {\bf nogood}, the consequent and the antecedents 
of an undermining argument may not be assumed to be true simultaneously.
Unlike an element of the set nogood, an undermining argument describes which
element has to be removed from the set of premisses (assumptions).

Because supporting arguments and labels are closely related, it is possible to describe
an ATMS using a propositional  logic.
Let $\langle A,N,J \rangle$ be an ATMS where:
\begin{itemize}
\item
$A$ is a set of assumptions,
\item
$N$ is a set of nodes, and
\item
$J$ is a set of justifications.
\end{itemize}
We can model the ATMS in the  logic using the following construction.
Let $A \cup N$ be the set of atomic propositions of the logic.
Furthermore, let the set of premisses $\Sigma$ be equal to $A \cup J$ where the 
justifications $J$ are described by rules of the form:
\[ p_1 \wedge ... \wedge p_n \to q. \]
Finally, let every justification be more reliable than any assumption.
Then the set ${\cal R}$ is equal to the set of maximal (under the inclusion relation)
environments of an 
ATMS.
Furthermore, for any linear extension of the reliability relation, the label for a node 
$n \in N$ is equal to the set:
\[ \{ P \mid P \Rightarrow n \in {\cal A}^{\prec'}_\infty \mbox{ and for no 
$Q \Rightarrow n \in {\cal A}^{\prec'}_\infty$: }  Q \subset P \}. \]
The set of nogoods is equal to the set:
\[ \{ (P \cup \{ p \} ) \cap A \mid \begin{array}[t]{l} 
P \not\Rightarrow p \in {\cal A}^{\prec'}_\infty 
\mbox{ and for no $Q \not\Rightarrow q \in {\cal A}^{\prec'}_\infty$: } \\
(Q \cup \{q\}) \cap A \subset (P \cup \{p\}) \cap A \}. 
\end{array} \]

\section{Applications}
In the previous sections a logic for reasoning with inconsistent
knowledge was described.
In this section two applications will be discussed.

\subsubsection*{Unreliable  knowledge sources}
In situations where we must deal with sensor data the logic described in
the previous sections can be applied.
To be able to reason with sensor data, the data has to be 
translated into statements about the world.
Because of measurements errors and of misinterpretation of the data,
these statements can be incorrect.
This may  result in inconsistencies.
These inconsistencies may be resolved by considering the reliability of 
the knowledge sources used.
To illustrate this consider the following example.
\begin{example}
Suppose that we want to determine the type of an airplane by using the
characteristic of its radar reflection.
The radar reflection of an airplane depends on the size and the 
shape of plane.
Suppose that we have some pattern recognition system that outputs
a proposition stating the type of plane, or a disjunction of possible
types in case of uncertainty.
Furthermore, suppose that we have an additional system that determines
the speed and the course of the plane.
The output of this system will also be stated as a proposition.
Given the output of the two systems, we can verify whether they are 
compatible.
If a plane is recognized as a Dakota and its speed is 1.5 Mach, then, 
knowing that a Dakota cannot go through the sound barrier, we can derive 
a conflict.
Since the speed measuring system is more reliable than the
type identifying system, we must remove the proposition stating
that the plane is a Dakota.
\end{example}

In this example, the reliability relation can be interpreted
as denoting that if two premisses are involved in a conflict the least reliable premiss has the highest probability of being wrong.
Since the relative probability is conditional on inconsistencies,
information from one reliable source cannot be overruled by information from many unreliable knowledge sources.
For example, the position of an object determined by seeing it is normally
more reliable then the position determined hearing it, independent of the 
number persons that heard it at some position.
Notice that fault probabilities have no meaning because faults are context
dependent.
The positions where you hear an object can be incorrect because of reflections and the limited speed of sound.
Usually, these factors cannot be predicted in advance.

\subsubsection*{Planning}
Another possible application for the logic can be found in the 
area of planning.
In \cite{Gin-88}, Ginsberg and Smith describe a possible worlds 
approach for reasoning about actions.
What they propose is an alternative way of determining
the consequences of an action. 
Instead of using frame axioms,
default rules, or add and delete lists. 
They propose to determine the nearest {\em world}
that is consistent with the consequences of an action.
The advantage of this approach is that we do not have to know all
possible consequences of an action in advance.
For example, in general, we cannot know whether putting a plant
on a table will obscure a picture on the wall.
Hence, if we know that a picture is not obscured before an action,
we may assume that it is still not obscured after the action when this 
fact is consistent with the consequences of the action.
\begin{figure}[ht]
\centerline{\includegraphics[scale=1]{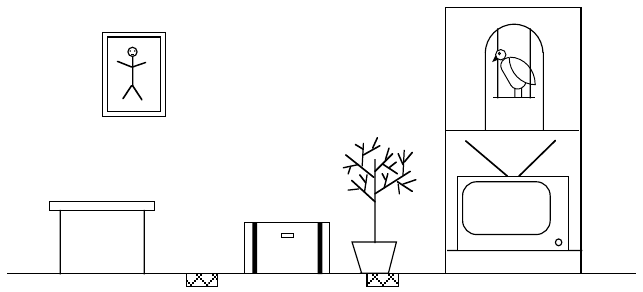}}
\caption{living-room} \label{room}
\end{figure}
\begin{example}
Figure \ref{room} can be described a set of premisses.
This set of premisses is divided in to three subsets, viz.\
the domain constraints, the structural facts and the remaining facts.
The domain constraints are:
\begin{enumerate}
\item ${\it on}(x,y) \wedge y \not= z \to \neg on (x,z)$
\item ${\it on}(x,y) \wedge z \not= x \wedge y \not= {\it floor} \to \neg on (z,y)$
\item ${\it rounded}(x) \to \neg on (x,y)$
\item ${\it duct}(d) \wedge \exists x. {\it on}(x,d) \to {\it blocked}(d)$
\item $\exists x. {\it on}(x,{\it table}) \leftrightarrow {\it obscured}({\it picture})$
\item ${\it blocked}({\it duct1}) \wedge {\it blocked}({\it duct2}) \leftrightarrow {\it stuffy}({\it room})$
\end{enumerate}
The structural facts are:
\begin{enumerate}
\setcounter{enumi}{6}
\item ${\it rounded}({\it bird})$
\item ${\it rounded}({\it plant})$
\item ${\it duct}({\it duct1})$
\item ${\it duct}({\it duct2})$
\item ${\it in}({\it bottom\_shelf}, {\it bookcase})$
\item ${\it in}({\it top\_shelf}, {\it bookcase})$
\end{enumerate}
The situational facts are:
\begin{enumerate}
\setcounter{enumi}{12}
\item ${\it on}({\it bird},{\it top\_shelf})$
\item ${\it on}(tv,{\it bottom\_shelf})$
\item ${\it on}({\it chest}, {\it floor})$
\item ${\it on}({\it plant}, {\it duct2})$
\item ${\it on}({\it bookcase}, {\it floor})$
\item ${\it blocked}({\it duct2})$
\item $\neg {\it obscured}({\it picture})$
\item $\neg {\it stuffy}({\it room})$
\end{enumerate}
\end{example}
The domain constraints are complemented with the {\em unique name assumption} (UNA). 

Clearly, the situational facts are less reliable than the structural facts
and the domain constraints.
Furthermore, facts added by recent actions are on average more reliable than 
facts added by less recent actions.

Now suppose that we move the {\it plant} from {\it duct2} to the {\it table}.
This can be described by adding the situational fact $on({\it plant}, {\it table})$.
From the new set of premisses we can derive two inconsistencies;
\begin{eqnarray*}
\lefteqn{\{ \exists x. {\it on}(x,{\it table}) \leftrightarrow {\it obscured}({\it picture}), } \\
 & & \neg {\it obscured}({\it picture}), {\it on}({\it plant}, {\it table}) \} 
\end{eqnarray*} and 
\begin{eqnarray*}
\lefteqn{\{ {\it on}(x,y) \wedge y \not= z \to \neg on (x,z), } \\
 & & {\it on}({\it plant}, {\it duct2}), {\it on}({\it plant}, {\it table}) \}. 
\end{eqnarray*}
The least reliable premisses in these sets of premisses are
respectively the facts $\neg {\it obscured}({\it picture})$ and ${\it on}({\it plant}, {\it duct2})$.
Hence, they have to be removed from the set of premisses.

\section{Conclusions}
In this paper a logic for reasoning with inconsistent knowledge has been described.
One of the original motivations for developing this logic was based on the view that
default reasoning is as a special case of reasoning with inconsistent
knowledge.
To describe defaults in this logic, such as Poole's framework for default reasoning,
formulas containing free variables can be used.
These formulas denote a set of ground instances.
If we do not generate these ground instances, but, by using unification of terms
containing free variables, we reason with formulas containing free variables, 
we can derive conclusions representing sets of instances.
This would seem to be a very useful property.

Since, in the logic described here a default rule can only be described by
using  material implication, a default rule has a contraposition. 
It is possible, however, the contraposition may not hold for default rules.
For example,  the contraposition of the default rule: `someone who owns a 
driving licence, may drive a car' is not valid.
A better candidate for a default reasoning would be Reiter's Default 
logic \cite{Rei-80} or Brewka's aproach \cite{Bre-91}.

Although it is likely that the logic is not suited for default reasoning, it
is suited for reasoning with knowledge coming from different and not fully reliable 
knowledge sources.
For this use of the logic, it seems plausible that the logic satisfies the 
properties of system {\bf P}.
As was shown in the examples described in Section 10, the reliability relation can be given a plausible probabilistic and ontological 
interpretations.
Furthermore, the current belief set with respect to the inferences made
can be determined efficiently.
One important disadvantage is that, given a set of premisses containing many inconsistencies and insufficient knowledge about the 
relative reliability, the number of possible belief sets can grow 
exponentially in the number minimal inconsistencies detected.

\section*{Appendix}
\setcounter{stelling}{11}
\begin{theorem} {\it Soundness} \\
For each $i \geq 0$: 
\begin{list}{}{} \item
if $P \Rightarrow \varphi
\in {\cal A}^{\prec'}_i$, then:
\[ P \subseteq {\Sigma} \mbox{ and } P
\models \varphi. \]
\end{list}
\end{theorem}
\begin{proof}
By the soundness of propositional logic, 
\begin{list}{}{} \item
if $P
\vdash \varphi $, then $P\models \varphi $. 
\end{list}
Therefore, we only have to prove that for each $i \geq 0$: 
\begin{list}{}{} \item
if $P \Rightarrow \varphi
\in {\cal A}^{\prec'}_i$, then 
$ P \subseteq {\Sigma} \mbox{ and } P
\vdash \varphi. $ 
\end{list}
We can prove this by induction on the index $i$ of ${\cal A}^{\prec'}_i$. 
\begin{itemize}
\item
For $i=0$: 
\begin{list}{}{} \item
$\{ \varphi \} \Rightarrow \varphi 
\in {\cal A}^{\prec'}_0$ if and only if 
$\varphi \in \Sigma $. 
\end{list}
Therefore, 
$\{ \varphi \} \vdash \varphi. $
\item
Proceeding inductively, suppose that $P \Rightarrow \varphi \in {\cal A}^{\prec'}_{i}$. \\
Then: 
\begin{list}{}{} \item
$P \Rightarrow \varphi \in {\cal A}^{\prec'}_{i}$ if and only if 
$P \Rightarrow \varphi \in {\cal A}^{\prec'}_{i-1}$ or $P \Rightarrow \varphi$
has been added by Rule~\ref{rule1} or~\ref{rule2}.
\end{list}
\begin{itemize}
\item
If $P \Rightarrow \varphi  
\in {\cal A}^{\prec'}_{i-1}$, then, by the induction hypothesis, 
\[ P \subseteq { \Sigma } \mbox{ and } P \vdash \varphi. \] 
\item
If $P \Rightarrow \varphi $ is introduced
by Rule~\ref{rule1}, then it is an axiom. \\
Therefore, $P = \emptyset$ and $ \vdash \varphi. $
\item
If $P \Rightarrow \varphi $ is introduced 
by Rule~\ref{rule2}, then there is a $Q \Rightarrow \psi
\in {\cal A}^{\prec'}_{i-1}$, $R \Rightarrow ( \psi 
\rightarrow \varphi )\in {\cal A}^{\prec'}_{i-1}$. \\
Therefore, $ P = (Q \cup R)$. \\
According to the induction hypothesis there holds:
 \[ Q, R \subseteq { \Sigma }, \]
\[ Q \vdash \psi \] and 
\[R \vdash  \psi \rightarrow \varphi . \]
Hence: \[ P \subseteq { \Sigma } \mbox{ and } P \vdash \varphi. \]
\end{itemize}
\end{itemize}
\end{proof}
\begin{theorem} {\it Completeness} \\
For each $P \subseteq {\Sigma}$: 
\begin{list}{}{} \item
if $P \models \varphi$, then there exists a $Q \subseteq P$ such that
for some $i \geq 0$: \[ Q \Rightarrow \varphi \in {\cal A}^{\prec'}_i. \] 
\end{list}
\end{theorem}
\begin{proof}
Let $P \subseteq {\Sigma}$ and $P \models \varphi$. \\
By the completeness of propositional logic, 
\begin{list}{}{} \item
if $P \models \varphi $, then $P \vdash \varphi$. 
\end{list}
Since $P \vdash \varphi$, there exists a deduction sequence
$\langle \varphi_1 , \varphi_2 , ... , \varphi_n \rangle$
such that $\varphi_n = \varphi$ and for each $j < n$: either
\begin{itemize}
\item
$\varphi_j \in P$, or
\item
$\varphi_j$ is an axiom, or 
\item
there exists a $\varphi_k$ and a $\varphi_l$ with $k,l < j$ and 
$\varphi_l = \varphi_k \rightarrow \varphi_j$.
\end{itemize}
The theorem will be proven, using induction on the length $n$ of the 
deduction sequence.
\begin{itemize}
\item
For $n=1$,  $\langle \varphi_1 \rangle$
is the deduction sequence for $P \vdash \varphi$.
\begin{itemize}
\item
If $\varphi_1 \in P$, then $\{ \varphi_1 \} \Rightarrow 
\varphi_1 \in {\cal A}^{\prec'}_0$. 
\item
If $\varphi_1$ is an axiom, then there exists some $i_0 \geq 1$ such that:
\begin{list}{}{} \item
${\cal A}^{\prec'}_{i_1} = {\cal A}^{\prec'}_{i_1 - 1} \cup \{ \emptyset \Rightarrow \varphi_1 \}$ and 
$\emptyset \Rightarrow \varphi_1$ is added by Rule~\ref{rule1}.
\end{list}
\end{itemize}
Hence, the theorem holds for deduction sequences of length 1.
\item
Proceeding inductively,
let $\langle \varphi_1 , \varphi_2 , ... , \varphi_{n} \rangle$
be a deduction sequence for \\ $P {\vdash} \varphi_{n}$.
\begin{itemize}
\item
If $\varphi_{n} \in P$, then $\{ \varphi_{n} \} \Rightarrow 
\varphi_{n} \in {\cal A}^{\prec'}_0$. 
\item
If $\varphi_{n}$ is an axiom, then there exists an $i_{n}$ such that:
\begin{list}{}{} \item
${\cal A}^{\prec'}_{i_{n}} = {\cal A}^{\prec'}_{i_{n} - 1} \cup \{ \emptyset \Rightarrow \varphi_{n} \}$
and $\emptyset \Rightarrow \varphi_{n}$ is added by Rule~\ref{rule1}.
\end{list}
\item
If there exists a $\varphi_k$ and a $\varphi_l$ with $k,l < n$ and
$\varphi_l = \varphi_k \rightarrow \varphi_{n}$, then, by the 
induction hypothesis, there exists some $i_k$ and some $i_l$ such that:
\[ R \Rightarrow \varphi_k \in {\cal A}^{\prec'}_{i_k}, \] 
\[ S \Rightarrow ( \varphi_k \rightarrow \varphi_{n} ) \in {\cal A}^{\prec'}_{i_l} \]
and \[ R,S \subseteq P. \]
Because of the fairness Assumption~\ref{fair}, there must exist an
$i_{n}$ with $i_k, i_l < i_{n}$ 
such that: 
\begin{list}{}{} \item
	${\cal A}^{\prec'}_{i_{n}} = {\cal A}^{\prec'}_{i_{n} - 1} \cup \{ R \cup S \Rightarrow \varphi_{n} \}$
	and $R \cup S \Rightarrow \varphi_{n}$ is added by Rule~\ref{rule2}.
\end{list}
\end{itemize}
Hence there exists some $i_{n}$  
such that $Q \Rightarrow \varphi_{n} \in {\cal A}^{\prec'}_{i_{n}}$ and $Q \subseteq P$.
\end{itemize}
\end{proof}
\begin{theorem} {\it Soundness} \\
For each $i \geq 0$: 
\begin{list}{}{} \item
if $P \not\Rightarrow \varphi 
\in {\cal A}^{\prec'}_i$, then: 
\[ (P \cup \{ \varphi \}) \subseteq 
{ \Sigma }, \mbox{ and } (P \cup \{ \varphi \}) \mbox{ is not satisfiable}. \]
\end{list}
\end{theorem}
\begin{proof}
The theorem will be proven using induction to the index $i$
of the set of arguments ${\cal A}^{\prec'}_i$.
\begin{itemize}
\item
For $i = 0$ the theorem holds vacuously, because there is no 
$P \not\Rightarrow \varphi  \in {\cal A}^{\prec'}_0$.
\item
Proceeding inductively,
suppose that $P \not\Rightarrow \varphi \in {\cal A}^{\prec'}_{i}$. \\
$P \not\Rightarrow \varphi  \in {\cal A}^{\prec'}_{i}$ 
if and only if $P \not\Rightarrow \varphi 
\in {\cal A}^{\prec'}_{i-1}$ or $P \not\Rightarrow \varphi $ 
has been added by Rule~\ref{rule3}.
\begin{itemize}
\item
If $P \not\Rightarrow \varphi  \in {\cal A}^{\prec'}_{i-1}$, 
then, by the induction hypothesis, 
$ (P \cup \{ \varphi \}) \subseteq { \Sigma }$ and 
$(P \cup \{ \varphi \})$ is not satisfiable.
\item
If $P \not\Rightarrow \varphi$ is 
introduced by Rule~\ref{rule3}, then there exists an $R \Rightarrow 
\psi \in {\cal A}^{\prec'}_{i-1}$ and a $Q \Rightarrow \neg \psi \in {\cal A}^{\prec'}_{i-1}$ such that:
\[ \varphi = \min_{\prec'}(Q \cup R) \mbox{ and } P = ( R \cup Q ) \backslash \varphi . \] 
By Theorem~\ref{th1}:
\[ R, Q \subseteq { \Sigma }, \]
\[ R \vdash \psi \mbox{ and } Q \vdash \neg \psi. \]
Hence $ (P \cup
\{ \varphi \} ) \subseteq { \Sigma }, $ and
$ (P \cup \{ \varphi \})$ is inconsistent. \\
Since inconsistency implies unsatisfiability: 
\[ (P \cup \{ \varphi \} ) \subseteq { \Sigma } \mbox{ and } 
(P \cup \{ \varphi \}) \mbox{ is not satisfiable}. \]
\end{itemize}
\end{itemize}
\end{proof}
\begin{theorem} {\it Completeness} \\
For each $P \subseteq { \Sigma }$: 
\begin{list}{}{} \item
if $P$ is a minimal 
unsatisfiable set of premisses and $\varphi  = \min_{\prec'}(P)$, then for some  
$i \geq 0$: 
\[ P \backslash \varphi \not\Rightarrow \varphi \in {\cal A}^{\prec'}_i. \]
\end{list}
\end{theorem}
\begin{proof}
Let $P$ be a minimal unsatisfiable subset of ${\Sigma}$ with
$\varphi = \min_{\prec'}(P)$. \\
Since $P$ is a minimal unsatisfiable set, $P$ is a minimal inconsistent set. \\
Therefore, there exists a proposition $\psi$ such that:
\[ P \vdash \psi \mbox{ and } P \vdash \neg \psi. \]
By Theorem~\ref{th2} there exists a  $j,k \geq 0$ such that:
\[ S \Rightarrow \psi \in {\cal A}^{\prec'}_j, S \subseteq P \] and 
\[ T \Rightarrow \neg \psi \in {\cal A}^{\prec'}_k, T \subseteq P. \] 
Hence, $(S \cup T) \subseteq P. $ \\
Since $P$ is a minimal inconsistent set of premisses: \[ (S \cup T) = P. \]
Because of the fairness Assumption \ref{fair} 
there exists an $i >j, k$ such that:
\[ (P \backslash \varphi) \not\Rightarrow \varphi \in {\cal A}^{\prec'}_i. \]
\end{proof}
\setcounter{stelling}{17}
\begin{property}
For every $i$, the set $\Delta^{\prec'}_i$ exists and is unique.
\end{property}
\begin{proof}
\begin{description}
\item[Existence]
Let $\delta_0 \supset \delta_1 \supset ... \supset \delta_k$
be a sequence of sets of premisses such that:
\begin{itemize}
\item
$\Sigma = \delta_0$,
\item
$\delta_{j+1} = \delta_j \backslash \{ \varphi \}$ where $\varphi$ is the most reliable 
premiss in $\delta_j$ such that \\ $P \not\Rightarrow \varphi \in {\cal A}^{\prec'}_i$ and $P \subseteq \delta_j$.
\end{itemize}
Then, by induction on the index of the sequence, we can prove that:
\[ \Sigma \backslash {\it Out}^{\prec'}_i(\delta_j) \subseteq \delta_j. \]
\begin{itemize}
\item
For $j =0$, clearly, there holds $\Sigma \backslash {\it Out}^{\prec'}_i(\delta_0) \subseteq \delta_0$.
\item
Proceeding inductively, let the induction hypothesis hold for
$\ell \leq j$. \\
If $\Sigma \backslash {\it Out}^{\prec'}_i(\delta_j) \subset \delta_j$, then there exists a most reliable
$\varphi \in \delta_j$ such that $P \not\Rightarrow \varphi$ and $P \subseteq \delta_j$.

Now suppose that $\Sigma \backslash {\it Out}^{\prec'}_i(\delta_{j+1}) \not\subseteq \delta_{j+1}$. \\
Then there exists a $\psi \not\in {\it Out}^{\prec'}_i(\delta_{j+1})$ and $\psi \not\in \delta_{j+1}$.

Suppose that $\psi \in \delta_j$. \\
Then $\psi = \varphi$. \\
Since $\varphi$ is the most reliable premiss such that 
$P \not\Rightarrow \varphi$ and $P \subseteq \delta_j$, $P \subseteq \delta_{j+1}$. \\
Hence, $\psi \in {\it Out}^{\prec'}_i(\delta_{j+1})$. \\
Contradiction.

Hence, $\psi \not\in \delta_j$ and, by the construction of $\delta_j$, $\varphi \prec' \psi$. \\
Since $\psi \not\in \delta_j$, by the induction hypothesis, $\psi \in {\it Out}(\delta_j)$. \\
Therefore, there exists a $Q \not\Rightarrow \psi \in {\cal A}^{\prec'}_i$ and $Q \subseteq \delta_j$. \\
Since $\varphi \prec' \psi$, $Q \subseteq \delta_{j+1}$. \\
Hence, $\psi \in {\it Out}^{\prec'}_i(\delta_{j+1})$. \\
Contradiction.

Hence, $\Sigma \backslash {\it Out}^{\prec'}_i(\delta_{j+1}) \subseteq \delta_{j+1}$.
\end{itemize}
Let $k$ be the highest index in the sequence. \\
Then there does not exist a $\varphi \in \delta_k$ such that 
$P \not\Rightarrow \varphi \in {\cal A}^{\prec'}_i$ and $P \subseteq \delta_k$. \\
Hence, $\Sigma \backslash {\it Out}^{\prec'}_i(\delta_k) = \delta_k$, otherwise there would  exist
a $\varphi \in \delta_k$ such that 
$P \not\Rightarrow \varphi \in {\cal A}^{\prec'}_i$ and $P \subseteq \delta_k$. 

Hence, there exists at least one $\Delta^{\prec'}_i$ such that:
\[ \Delta^{\prec'}_i = \Sigma \backslash {\it Out}^{\prec'}_i(\Delta^{\prec'}_i). \]
\item[Uniqueness]
Suppose $\Delta^{\prec'}_i$ is not unique. \\
Then there exist at least two different subsets 
$\Delta^{\prec'}_i, \Delta'^{\prec'}_i \subset{\Sigma}$ satisfying Definition \ref{bel.prem}. \\
Let $\varphi$ be the most reliable proposition in $\Sigma$ such that:
\[ \varphi \not\in \Delta^{\prec'}_i \mbox{ and } \varphi \in \Delta'^{\prec'}_i. \] or 
\[ \varphi \not\in \Delta'^{\prec'}_i \mbox{ and } \varphi \in \Delta^{\prec'}_i. \]
Let us consider the first case. Note that the second case is similar. Then, there exists a $P \not\Rightarrow \varphi \in {\cal A}^{\prec'}_i$. \\
By Theorem~\ref{th4} there holds: 
\begin{list}{}{} \item
$P \cup \{ \varphi \}$ is unsatisfiable. 
\end{list}
Therefore, there exists a minimal inconsistent set of premisses $Q$ 
with \\ $\varphi =\min_{\prec'}(Q)$. \\
Since $\varphi \not\in \Delta^{\prec'}_i $ and $\varphi \in \Delta'^{\prec'}_i$, there exists a
$\psi \in Q$ such that: 
\[ \psi \in \Delta^{\prec'}_i, \psi \not\in \Delta'^{\prec'}_i \mbox{ and } \varphi \prec \psi. \]
Hence, $\varphi$ is not the most reliable proposition in $\Sigma$ such that:
\[ \varphi \not\in \Delta^{\prec'}_i \mbox{ and } \varphi \in \Delta'^{\prec'}_i. \] or 
\[ \varphi \not\in \Delta'^{\prec'}_i \mbox{ and } \varphi \in \Delta^{\prec'}_i. \]
Contradiction.

Hence $\Delta^{\prec'}_i$ is unique.
\end{description}
\end{proof}
\setcounter{stelling}{19}
\begin{property}
$ \mbox{For each } \varphi \in B^{\prec'}_i$: $\Delta^{\prec'}_i \vdash \varphi$.
\end{property}
\begin{proof}
Suppose $\varphi \in B^{\prec'}_i$. \\
Then there exists a $ P \Rightarrow 
\varphi \in {\cal A}^{\prec'}_i $ such that: \[ P \subseteq \Delta^{\prec'}_i. \]
Therefore, by Theorem~\ref{th1}: \[ P \vdash \varphi \mbox{ and } P \subseteq \Delta^{\prec'}_i. \]
Hence, $\Delta^{\prec'}_i \vdash \varphi. $
\end{proof}
\setcounter{stelling}{21}
\begin{property} 
$\Delta^{\prec'}_\infty$ is maximal consistent.
\end{property}
\begin{proof}
Suppose that $\Delta^{\prec'}_\infty$ is inconsistent. \\
Then there exists a minimal inconsistent subset $M$ of $\Delta^{\prec'}_\infty$. \\
Let $\varphi = \min_{\prec'}(M)$. \\
Then by Theorem~\ref{th4} there exists an $i$ with 
\[ P \not\Rightarrow \varphi \in {\cal A}^{\prec'}_i \]
Hence $P \not\Rightarrow \varphi \in {\cal A}^{\prec'}_\infty$. \\
Because $P \subseteq \Delta^{\prec'}_\infty$, $\varphi \not\in 
\Delta^{\prec'}_\infty$. \\
Contradiction. 

Suppose that some $\Delta^{\prec'}_\infty$ is not maximal consistent. \\
Then there exists a $\varphi \in ( {\Sigma} \backslash \Delta^{\prec'}_\infty )$ and
$\{\varphi\} \cup \Delta^{\prec'}_\infty$ is consistent. \\
Since $\varphi \in ( {\Sigma} \backslash \Delta^{\prec'}_\infty )$,
$\varphi \in {\it Out}^{\prec'}_\infty ( \Delta^{\prec'}_\infty)$. \\
Therefore, there exists a $P \not\Rightarrow \varphi \in {\cal A}^{\prec'}_\infty$
and $P \subseteq \Delta^{\prec'}_\infty$. \\
Since $P \not\Rightarrow \varphi \in {\cal A}^{\prec'}_\infty$, $P \cup \{ \varphi \}$
is inconsistent. \\
Hence $\Delta^{\prec'}_\infty \cup \{ \varphi \} $ is inconsistent. \\
Contradiction.
\end{proof}
\begin{property} \mbox{}
\[ B_\infty = {\it Th}(\Delta^{\prec'}_\infty) \]
where $ {\it Th}(S) = \{ \varphi \mid S \vdash\varphi \} $
\end{property}
\begin{proof}
According to Property \ref{belief}:
\begin{list}{}{} \item
if $\varphi \in B_\infty$, then $\Delta^{\prec'}_\infty \vdash \varphi$.
\end{list}
Suppose there exists a $\varphi$ such 
that: \[ \varphi \not\in B_\infty \mbox{ and }
\varphi \in  {\it Th}(\Delta^{\prec'}_\infty). \]
Since $\varphi \in  {\it Th}(\Delta^{\prec'}_\infty)$,
$ \Delta^{\prec'}_\infty \vdash \varphi$. \\
By Theorem~\ref{th2} there exists some $i$ 
and some $P \Rightarrow \varphi \in {\cal A}^{\prec'}_i$ such that:
\[ P \subseteq \Delta^{\prec'}_\infty. \]
Therefore, there exists a
$P \Rightarrow \varphi \in {\cal A}^{\prec'}_\infty$ such that:
\[ P \subseteq \Delta^{\prec'}_\infty. \]
Hence, by Definition \ref{belief}: \[ \varphi \in B_\infty. \] 
Contradiction. 

Hence $B_\infty = {\it Th}( \Delta^{\prec'}_\infty )$.
\end{proof}
\begin{theorem}
Let $\langle \Sigma,\prec \rangle$ 
be a reliability theory.

Then there holds:
\[ {\cal R} = \{ \Delta^{\prec'}_\infty \mid \mbox{ for some linear extension $\prec'$ of $\prec$,
$\Delta^{\prec'}_\infty$ can be derived} \}. \]
\end{theorem}
\begin{proof}
Let $\Delta^{\prec'}_\infty$ be a set of believed premisses given a linear 
extension $\prec'$ of $\prec$. \\
Furthermore, let $\sigma_1,..., \sigma_m$ be an enumeration of $\Sigma$
such that for every $\sigma_j \prec' \sigma_k $: $k < j$. \\
Clearly, given this enumeration of $\Sigma$, $\Delta^{\prec'}_\infty$ will satisfy 
Definition \ref{proof-th}.

Let $D$ be a most reliable consistent set of premisses according to Definition \ref{proof-th} given an enumeration
$\sigma_1,...,\sigma_n$ of $\Sigma$.
Furthermore, let $\prec'$ be a linear extension of $\prec$ such that for
each $k <j$: $\sigma_j \prec' \sigma_k$. \\
%

Now suppose that: 
\[ D \not= {\Sigma} \backslash {\it Out}^{\prec'}_\infty (D). \] 
Hence, there exists a most reliable premiss $\varphi \in \Sigma$ such that either 
\[ \varphi \in D \mbox{ and } \varphi \in  {\it Out}^{\prec'}_\infty (D) \]
or \[ \varphi \not\in D \mbox{ and } \varphi \not\in  {\it Out}^{\prec'}_\infty (D). \]

Suppose that $\varphi \in D$ and $\varphi \in  {\it Out}^{\prec'}_\infty (D)$. \\
Since $\varphi \in  {\it Out}^{\prec'}_\infty (D)$, for some $P \not\Rightarrow \varphi \in {\cal A}^{\prec'}_\infty$
there holds:
\[ P \subseteq D. \]
Since $P \subseteq D$ and since $\varphi \in D$, $D$ is inconsistent. \\
By Definition \ref{proof-th}, however, $D$ must be consistent. \\
Contradiction. 

Suppose that $\varphi \not\in D$ and $\varphi \not\in  {\it Out}^{\prec'}_\infty (D)$. \\
Since $\varphi \not\in D$,  there exists a
minimal inconsistent of premisses $\{ \sigma_{i_1},...,\sigma_{i_l} \}$ where $i_j$ are indexes of the enumeration of $\Sigma$, $i_j < i_{j-1}$, $\{ \sigma_{i_1},...,\sigma_{i_{l-1}} \} \subseteq D$ and $\varphi = \sigma_{i_l}$. \\
Therefore, by Theorem \ref{th4} an by the above given definition of $\prec'$: 
\[ \{ \sigma_{i_1},...,\sigma_{i_{l-1}} \} \not\Rightarrow \sigma_{i_l} \in {\cal A}^{\prec'}_\infty \]
Hence, $\varphi \in {\it Out}^{\prec'}_\infty (D)$ \\
Contradiction.
\end{proof}

\setcounter{stelling}{29}

\begin{property}
	Let $\langle {\Sigma},\prec \rangle$ be a reliability theory and let $\sqsubset$ be the preference relation over interpretations defined the reliability theory.
	
	$\sqsubset$ is irreflexive and transitive.
\end{property}
\begin{proof}
	Suppose the $\sqsubset$ is not irreflexive. \\
	Then for some interpretations $\cal M, N$, $\cal M \sqsubset N$ and $\cal N \sqsubset M$. \\
	Since ${\it Prem}({\cal M}) \not= {\it Prem}({\cal N})$, for some $\varphi$, $\varphi \in {\it Prem}({\cal M}) \backslash {\it Prem}({\cal N})$ or $\varphi \in {\it Prem}({\cal N}) \backslash {\it Prem}({\cal M})$. 
	
	Consider the former case. \\
	Let $\varphi$ be a most reliable premiss such that $\varphi \in {\it Prem}({\cal M}) \backslash {\it Prem}({\cal N})$. \footnote{A most reliable premiss exist because $\Sigma$ is finite and $\prec$ is defined over $\Sigma$.} \\
	Since $\cal M \sqsubset N$, there is a $\psi \in {\it Prem}({\cal N}) \backslash {\it Prem}({\cal M})$ such that $\varphi \prec \psi$. \\
	Since $\psi \in {\it Prem}({\cal N}) \backslash {\it Prem}({\cal M})$ and $\cal N \sqsubset M$, there is an $\eta \in {\it Prem}({\cal M}) \backslash {\it Prem}({\cal N})$ such that $\psi \prec \eta$. \\
	Since $\varphi \prec \psi \prec \eta$, $\varphi$ cannot be a most reliable premiss such that $\varphi \in {\it Prem}({\cal M}) \backslash {\it Prem}({\cal N})$. \\
	Contradiction.
	
	The latter case is similar and also results in a contradiction.
	
	Hence, $\sqsubset$ is irreflexive.
	
	Suppose that $\sqsubset$  is not transitive. \\
	Then there exist structures $\cal L, M, N$ such that $\cal  L \sqsubset M \sqsubset N$ but $\cal  L \not\sqsubset N$. \\
	Therefore, for some $\varphi \in {\it Prem}({\cal L}) \backslash {\it Prem}({\cal N})$ there is no $\psi \in {\it Prem}({\cal N}) \backslash {\it Prem}({\cal L})$ such that $\varphi \prec \psi$.
	Let $\varphi$ be the most reliable premiss for which the above holds. \\
	The following cases can be distinguished:
	\begin{itemize}
		\item 
		Suppose that $\varphi \in  {\it Prem}({\cal M})$. \\ 
		Then $\varphi \in {\it Prem}({\cal M}) \backslash {\it Prem}({\cal N})$ and therefore, there is an $\eta \in {\it Prem}({\cal N}) \backslash {\it Prem}({\cal M})$ such that $\varphi \prec \eta$. \\
		Let $\eta$ be the most reliable premiss such that  $\eta \in {\it Prem}({\cal N}) \backslash {\it Prem}({\cal M})$ and $\varphi \prec \eta$.
		
		\begin{itemize}
			\item[]		
			Suppose $\eta \not\in {\it Prem}({\cal L})$. \\
			Then there is a $\psi \in {\it Prem}({\cal N}) \backslash {\it Prem}({\cal L})$, namely $\psi= \eta$, such that $\varphi \prec \psi$. \\
			Contradiction.
			
			Hence, $\eta \in {\it Prem}({\cal L})$. \\
			Therefore, $\eta \in {\it Prem}({\cal L}) \backslash {\it Prem}({\cal M})$, implying that there is a \\ $\mu \in {\it Prem}({\cal M}) \backslash {\it Prem}({\cal L})$ such that $\eta \prec \mu$. 
			
			Suppose that $\mu \in  {\it Prem}({\cal N})$. \\ 
			Then there is a $\psi \in {\it Prem}({\cal N}) \backslash {\it Prem}({\cal L})$, namely $\psi= \mu$, such that $\varphi \prec \psi$ because $\varphi \prec \eta \prec \mu$ and $\prec$ is transitively closed. \\
			Contradiction.
			
			Hence, $\mu \not\in {\it Prem}({\cal N})$. \\
			Therefore, there is a $\xi \in {\it Prem}({\cal N}) \backslash {\it Prem}({\cal M})$ such that $\mu \prec \xi$. \\
			Since $\varphi \prec \eta \prec \mu \prec \xi$ and $\prec$ is transitive, $\eta$ is not the most reliable premiss such that  $\eta \in {\it Prem}({\cal N}) \backslash {\it Prem}({\cal M})$ and $\varphi \prec \eta$. \\
			Contradiction.
		\end{itemize}
		
		\item 
		Suppose that $\varphi \not\in  {\it Prem}({\cal M})$. \\
		Then there is a $\eta \in {\it Prem}({\cal M}) \backslash {\it Prem}({\cal L})$ such that $\varphi \prec \eta$. \\
		Let $\eta$ be the most reliable premiss such that  $\eta \in {\it Prem}({\cal M}) \backslash {\it Prem}({\cal L})$ and $\varphi \prec \eta$.
		
		\begin{itemize}
			\item[]	
			Suppose that $\eta \in  {\it Prem}({\cal N})$. \\ 
			Then there is a $\psi \in {\it Prem}({\cal N}) \backslash {\it Prem}({\cal L})$, namely $\psi= \eta$, such that $\varphi \prec \psi$. \\
			Contradiction.
			
			Hence, $\eta \not\in {\it Prem}({\cal N})$. \\
			Therefore, there is a $\mu \in {\it Prem}({\cal N}) \backslash {\it Prem}({\cal M})$ such that $\eta \prec \mu$. 
			
			Suppose $\mu \not\in {\it Prem}({\cal L})$. \\
			Then there is a $\psi \in {\it Prem}({\cal N}) \backslash {\it Prem}({\cal L})$, namely $\psi= \mu$, such that $\varphi \prec \psi$ because $\varphi \prec \eta \prec \mu$ and $\prec$ is transitively closed. \\
			Contradiction. 
			
			Hence, $\mu \in {\it Prem}({\cal L})$. \\
			Therefore, there is a $\xi \in {\it Prem}({\cal M}) \backslash {\it Prem}({\cal L})$ such that $\mu \prec \xi$. \\
			Since $\varphi \prec \eta \prec \mu \prec \xi$ and $\prec$ is transitive, $\eta$ is not the most reliable premiss such that  $\eta \in {\it Prem}({\cal M}) \backslash {\it Prem}({\cal L})$ and $\varphi \prec \eta$. \\
			Contradiction.
		\end{itemize}
	\end{itemize}
	
	Hence, $\sqsubset$ is transitive.
\end{proof}

\setcounter{stelling}{31}

\begin{theorem}
Let $\langle {\Sigma},\prec \rangle$ be a reliability theory.
Furthermore, let ${\cal R}$ be the corresponding set of 
all most reliable consistent sets of premisses.
Then:
\[ {\it Mod}_\sqsubset (\langle {\Sigma},\prec \rangle) = \bigcup_{\Delta^{\prec'}_\infty \in {\cal R}} 
{\it Mod}( \Delta^{\prec'}_\infty) \]
where ${\it Mod}(S)$ denotes the set of classical models for a set of propositions $S$. 
\end{theorem}
\begin{proof}
The proof of \[ {\it Mod}_\sqsubset (\langle {\Sigma},\prec \rangle) = 
\bigcup_{\Delta^{\prec'}_\infty \in 
{\cal R}} {\it Mod}( \Delta^{\prec'}_\infty) \] can be divided into the proof of the 
soundness 
\[ {\it Mod}_\sqsubset (\langle {\Sigma},\prec \rangle) \subseteq \bigcup_{\Delta^{\prec'}_\infty \in {\cal R}} {\it Mod}( \Delta^{\prec'}_\infty) \]
and the proof of the completeness 
\[ \bigcup_{\Delta^{\prec'}_\infty \in {\cal R}} {\it Mod}( \Delta^{\prec'}_\infty) \subseteq {\it Mod}_\sqsubset (\langle {\Sigma},\prec \rangle) \]
of the logic.
\begin{description}
\item[Completeness]
Suppose that for some $\Delta^{\prec'}_\infty \in {\cal R}$ and some ${\cal M} \in
{\it Mod}(\Delta^{\prec'}_\infty)$: \[ {\cal M} \not\in {\it Mod}_\sqsubset (\langle {\Sigma},\prec \rangle). \]
Then there exists a structure $\cal N$: \[ \cal M \sqsubset N. \]
${\it Prem}({\cal M}) = \Delta^{\prec'}_\infty$ because $\Delta^{\prec'}_\infty$ is a maximal consistent set of premisses. \\
Therefore, according to Proposition \ref{consis}: 
\[ \Delta^{\prec'}_\infty \not\subset {\it Prem}({\cal N}). \]
Let $\varphi \in \Delta^{\prec'}_\infty$ be the\emph{ most reliable} premiss
according to the linear extension $\prec'$ of $\prec$, such that
$\varphi \in (\Delta^{\prec'}_\infty \backslash {\it Prem}({\cal N}))$. \\
Now by Definition \ref{prefM} 
there exists a $\psi \in ({\it Prem}({\cal N}) \backslash \Delta^{\prec'}_\infty)$ 
such that $\varphi \prec \psi$. \\
Since $\psi \not\in \Delta^{\prec'}_\infty$, there exists a $P \not\Rightarrow \psi \in {\cal A}^{\prec'}_\infty$
such that:  \[ P \subseteq \Delta^{\prec'}_\infty. \]
Now, $P \not\subseteq {\it Prem}({\cal N})$, otherwise ${\it Prem}({\cal N})$
would be inconsistent. \\
Hence, there exists a $\mu \in P$: 
\[  \mu \in (\Delta^{\prec'}_\infty \backslash {\it Prem}({\cal N})). \]
Since $P \not\Rightarrow \psi \in {\cal A}^{\prec'}_\infty$, $\psi \prec' \mu$. \\
Hence, $\varphi \prec' \psi \prec' \mu. $ \\
Contradiction. 

Hence, \[\bigcup_{\Delta^{\prec'}_\infty \in {\cal R}} {\it Mod}( \Delta^{\prec'}_\infty) \subseteq 
{\it Mod}_\sqsubset (\langle {\Sigma},\prec \rangle). \]

\item[Soundness]
Let ${\cal M} \in {\it Mod}_\sqsubset (\langle {\Sigma},\prec \rangle)$. \\
Then, ${\cal M} \in \bigcup_{\Delta^{\prec'}_\infty \in {\cal R}} {\it Mod}( \Delta^{\prec'}_\infty)$ has to be proven. \\
Note that ${\it Prem}({\cal M})$ is a maximal consistent set of premisses because otherwise, there would exists an $\cal N$ such that $\cal M \sqsubset N$. \\
Hence, for some linear extension $\prec'$ of $\prec$,  ${\it Prem}({\cal M}) = \Delta^{\prec'}_\infty$ has to be proven. 

The proof is based on constructing a linear extension $\prec'$ of $\prec$. \\
Staring from the most reliable premiss in ${\it Prem}({\cal M})$ given the reliability relation $\prec$, and an initial reliability relation $\prec^*_0 = \prec$, an extension  $\prec^*_{|{\it Prem}({\cal M})|}$ of $\prec$ is constructed. 

Let $\varphi \in {\it Prem}({\cal M})$ be a most reliable premiss given $\prec$ that has not been addressed yet, and let $\prec^*_i$ be the reliability relation constructed so far. \\
Create $\prec^*_{i+1}$ by adding $\eta \prec \varphi$ to $\prec^*_i$ for every minimal inconsistent subset $P$ of $\Sigma$ such that $\varphi \in P$, and for every $\eta \in P \backslash \varphi$ such that $\varphi \not\prec^*_i \eta$, and subsequently taking the transitive closure. 

The correctness of the of the reliability relation $\prec^*_{i+1}$ depends on the condition $\varphi \not\prec^*_i \eta$. There are three cases:
\begin{itemize}
	\item 
	Suppose $\eta \in  {\it Prem}({\cal M})$ and that $\eta$ is considered before $\varphi$. \\
	Then clearly, $\varphi \prec^*_i \eta$. Moreover, there must be a $\psi \in P$, $\psi \not\in {\it Prem}({\cal M})$ and the construction of $\prec^*$ guarantees that $\psi \prec^*_{i+1} \varphi$. \\
	Hence, the construction is correct.
	\item 
	Suppose $\eta \in  {\it Prem}({\cal M})$ and that $\eta$ is not considered before $\varphi$. \\
	Then, $\varphi \not\prec \eta$, and therefore, $\varphi \not\prec^*_i \eta$. \\
	Hence, the construction is correct.
	\item 
	Suppose $\eta \not\in  {\it Prem}({\cal M})$. \\
	Then there is a minimal inconsistent subset $Q$ of $\Sigma$ such that $\eta$ is a least preferred premiss $Q$ given $\prec$, and $Q \backslash \eta \subseteq {\it Prem}({\cal M})$. \\
	Hence, the construction is correct.
\end{itemize}

After constructing a reliability relation $\prec^*_{|{\it Prem}({\cal M})|}$, the final step is to take a linear extension of $\prec^*_{|{\it Prem}({\cal M})|}$ in order to get $\prec'$.
\end{description}
\end{proof}

\begin{lemma}
Let $\langle {\Sigma},\prec \rangle$ be a reliability theory.
Furthermore, let $\widehat{\alpha} = \{ {\cal M} \mid {\cal M} \models \alpha \} $,
let $\Sigma' = \Sigma \cup \{ \alpha \}$ and 
let $\prec' \; = ( \prec \cap \ ( \Sigma / \alpha \times 
\Sigma / \alpha ) )  \cup \{ \left< \varphi, \alpha \right> \mid
\varphi \in \Sigma / \alpha \}$.

Then $ {\cal M} \in {\it Mod}_{\sqsubset'} (\langle {\Sigma'},\prec' \rangle) $
if and only if
$ {\cal M} \in \widehat{\alpha} $
and for no ${\cal N} \in \widehat{\alpha}$: \[ {\cal M \sqsubset N}. \]
\end{lemma}
\begin{proof}
The results presented in the following two items, will be used in the proof.
\begin{itemize}
\item 
Suppose that ${\cal M} \in \widehat{\alpha}$ and ${\cal N} \not\in \widehat{\alpha}$,
i.e.\ ${\cal M} \models {\alpha}$ and ${\cal N} \not\models {\alpha}$.\\
Then by Definition \ref{prem}: \[ {\it Prem}({\cal M}) \not= {\it Prem}({\cal N}). \]
Therefore, \[ \alpha \in ({\it Prem}({\cal M}) \backslash {\it Prem}({\cal N})) \]
and for each $\varphi \in ({\it Prem}({\cal N}) \backslash {\it Prem}({\cal M}))$ 
there holds: \[ \varphi \prec' \alpha. \]

Hence by Definition \ref{prefM} for each ${\cal M} \in \widehat{\alpha}$ 
and ${\cal N} \not\in \widehat{\alpha}$:
\[ \cal N \sqsubset' M. \]
\item
Suppose that ${\cal M,N} \in \widehat{\alpha}$. \\
Since ${\cal M,N} \models \alpha$, for each $\varphi \in ( {\it Prem}({\cal M}) \backslash {\it Prem}({\cal N}))$
and for each \\ $\psi \in ({\it Prem}({\cal N}) \backslash {\it Prem}({\cal M}))$: 
\[ \varphi \prec \psi \mbox{ if and only if } \varphi \prec' \psi. \]

Hence, for each ${\cal M,N} \in \widehat{\alpha}$: \[ \cal N \sqsubset' M
\mbox{ if and only if } \cal N \sqsubset M. \]
\end{itemize}
Let $ {\cal M} \in {\it Mod}_{\sqsubset'} (\langle {\Sigma'},\prec' \rangle) $. \\
The first item above shows that ${\cal M} \in \widehat{\alpha}$ and the second item shows that for no ${\cal N} \in \widehat{\alpha}$: $\cal M \sqsubset N$.

Let ${\cal M} \in \widehat{\alpha}$ and for no ${\cal N} \in \widehat{\alpha}$: $\cal M \sqsubset N$. \\
The first item shows that for no ${\cal N} \not\in \widehat{\alpha}$: $\cal M \sqsubset' N$. \\
The second item shows that ${\cal M}$ is preferred in $\widehat{\alpha}$ given $\sqsubset'$.
\end{proof}

\begin{theorem}
	Let $\langle {\Sigma},\prec \rangle$ be a reliability theory. Moreover, let $\langle S,l,< \rangle$ be a triple where the set of states $S$ is the set of all possible interpretations for the language $L$, where $l : S \to S$ is the identity function, and where for each
	${\cal M,N} \in S$: 
	\begin{list}{}{}
		\item
		$\cal M < N$ if and only if $\cal N \sqsubset M$.
	\end{list}
	
	Then $\langle S,l,< \rangle$ is a \textit{preferential model} \cite{Kra-90}.
\end{theorem}

%
\begin{proof}
	Since $S$ is the set of all interpretations and since $l$ is the identity function, a state corresponds one to one to an interpretation. \\
	Therefore, since the relation $\sqsubset$ is an irreflexive and transitive partial order on interpretations, 
	so is $<$ on $S$.
	
A transitive relation $<$ implies smoothness if $S$ is finite. \\
In case of first order logic, there might exist an infinite long chain of preference, and therefore smoothness has to be proven.
	
Suppose that $<$ is not smooth. \\
Then by Lemma \ref{lem} for some proposition $\alpha$ and some
${\cal M} \in \widehat{\alpha}$ there holds neither that:
\[ {\cal M} \in {\it Mod}_{\sqsubset'}(\langle \Sigma',\prec' \rangle), \]
nor does there exist a $ {\cal N} \in {\it Mod}_{\sqsubset'} (\langle {\Sigma'},\prec' \rangle)$
such that: \[ \cal M \sqsubset N. \]
If ${\cal M} \not\in {\it Mod}_{\sqsubset'}(\langle \Sigma',\prec' \rangle)$,
there must exists an ${\cal L}_1$: ${\cal M} \sqsubset {\cal L}_1$. \\
Suppose that for some ${\cal L}_i$ with $i \geq 1$ 
there does not exist an ${\cal L}_{i+1}$ such that:
\[ {\cal L}_i \sqsubset {\cal L}_{i+1}. \]
Then ${\cal L}_i \in {\it Mod}_{\sqsubset'}(\langle \Sigma',\prec' \rangle)$. \\
According to Property \ref{prop:irref-trans}, 
\[ {\cal M} \sqsubset {\cal L}_{i+1}. \]
Contradiction.

Hence, there exists an infinite sequence ${\cal M} \sqsubset {\cal L}_1 
\sqsubset {\cal L}_2 \sqsubset ...$. 

For each ${\cal L}_i$ there exists a ${\it Prem}({\cal L}_i) \subseteq \Sigma$. \\
Suppose that for some $j <i$: ${\it Prem}({\cal L}_i) = {\it Prem}({\cal L}_j)$. \\
Then ${\cal L}_j \not\sqsubset {\cal L}_i$. \\
Since $\sqsubset$ is transitive, ${\cal L}_j \sqsubset {\cal L}_i$. \\
Contradiction.

Hence, for each ${\cal L}_i, {\cal L}_j$ with $i \not= j$: ${\it Prem}({\cal L}_i) \not= {\it Prem}({\cal L}_j)$. \\
Since $\Sigma$ is finite, for some $i, j$ with $i \not= j$: ${\it Prem}({\cal L}_i) = {\it Prem}({\cal L}_j)$. \\
Contradiction.


Hence, $<$ is smooth.

Hence, $\langle S,l,< \rangle$ is a preferential model according to the definition of
Kraus et al.
\end{proof}
\begin{theorem}
Let $W= \left< S,l,< \right>$ be a preferential model for $\langle {\Sigma},\prec \rangle$.
Then the following equivalence holds:
\begin{list}{}{} \item
$\alpha \nm_W \beta$ if and only if 
\[ \Sigma' = \Sigma \cup \{ \alpha \}, \] 
\[ \prec' \; = ( \prec \cap \ ( \Sigma / \alpha \times 
\Sigma / \alpha ) ) \cup \{ \left< \varphi, \alpha \right> \mid
\varphi \in \Sigma / \alpha \} \] 
and $\beta \in {\it Th}(\langle \Sigma', \prec' \rangle)$.
\end{list}
\end{theorem}
\begin{proof}
According to Theorem \ref{s.c}:
\begin{list}{}{} \item
$\beta \in {\it Th}(\langle \Sigma', \prec' \rangle)$ if and only if 
for each ${\cal M} \in {\it Mod}_{\sqsubset'}
(\langle {\Sigma'},\prec' \rangle)$: \[ {\cal M} \models \beta. \]
\end{list}
Therefore, by Lemma \ref{lem}:
\begin{list}{}{} \item
$\beta \in {\it Th}(\langle \Sigma', \prec' \rangle)$ if and only if for each ${\cal M}
\in \max_{\sqsubset}(\widehat{\alpha})$: \[ {\cal M} \models \beta. \]
\end{list}
Hence, by the definition of the non-monotonic entailment relation $\nm$ we have:
\begin{list}{}{} \item
$\beta \in {\it Th}(\langle \Sigma', \prec' \rangle)$ if and only if $\alpha \nm_W \beta$.
\end{list}
\end{proof}
\setcounter{stelling}{36}
\begin{theorem}
Let belief set $K = {\it Th}(\langle \Sigma, \prec \rangle)$ be the 
set of theorems of the reliability theory 
$\langle \Sigma, \prec \rangle$. \\
Suppose that $K^*[\alpha]$ is the belief set of the premisses 
$\Sigma \cup \{ \alpha \}$ with 
reliability relation: \[ \prec' \; = ( \prec \cap \ ( \Sigma / \alpha \times 
\Sigma / \alpha ) ) \cup \{ \left< \varphi, \alpha \right> \mid
\varphi \in \Sigma / \alpha \}; \]
i.e.\ $K^*[\alpha] = \{ \beta \mid \alpha \nm_W \beta \}$ where $W$
is a preferential model for 
$\langle {\Sigma},\prec \rangle$. \\
Then the following postulates are satisfied.
\begin{enumerate}
\item
$K^*[\alpha]$ is a belief set.
\item
$\alpha \in K^*[\alpha]$.
\setcounter{enumi}{5}
\item
If $\vdash \alpha \leftrightarrow \beta$, then $K^*[\alpha] = K^*[\beta]$.
\end{enumerate}
\end{theorem}
\begin{proof}
\hspace*{\fill} \\ \vspace{-6mm}
\begin{enumerate}
\item
This follows from Property \ref{Binf}
\item
Since $\alpha \nm_W \alpha$ (reflexivity), $\alpha \in K^* [\alpha]$.
\setcounter{enumi}{5}
\item
Since $\displaystyle \frac{\models \alpha \leftrightarrow \beta, \alpha \nm_W \gamma}{
\beta \nm_W \gamma}$ (left logical equivalence), 
if $\vdash \alpha \leftrightarrow \beta$, \\
then $K^*[\alpha] = K^*[\beta]$.
\end{enumerate}
\end{proof}

\section*{\rm Acknowledgement}
I thank C. Witteveen for the discussions we had on the content of paper.
Furthermore, I thank the reviewers. Due to their comments the paper
has improved significantly.
Finally, I thank the Delft University of Technology (TU-Delft) and 
and the National Aerospace Laboratory (NLR) for supporting
the research reported on.

\end{document}